\documentclass[5p,times]{elsarticle}

\usepackage{bm}
\usepackage{enumerate}
\usepackage{amsmath}
\usepackage{graphicx}
\usepackage{amsmath}
\usepackage{subcaption}
\usepackage{breqn}
\usepackage[table]{xcolor}
\usepackage{booktabs}
\usepackage{empheq}
\usepackage{amssymb}
\usepackage{hyperref}
\usepackage{xcolor}
\usepackage{amsthm}
\usepackage{float}

\newtheorem{theorem}{Theorem}
\newtheorem{corollary}{Corollary}
\newtheorem{lemma}{Lemma}

\newtheorem{assumption}{Assumption}

\newtheorem{remark}{Remark}

\newtheorem{problem}{Problem}

\newcommand{\redtext}[1]{{\color{black}#1}}

\newcommand{\bmh}[1]{\hat{\bm{#1}}} 
\newcommand{\bmdh}[1]{\dot{\hat{\bm{#1}}}}
 
\newcommand{\bmd}[1]{\dot{\bm{#1}}} 
\newcommand{\bmdd}[1]{\ddot{\bm{#1}}}

\newcommand{\bracketmat}[2]{ \left[ \begin{array}{#1} #2 \end{array} \right] }

\newcommand{\myvar}[1]{\bm{#1}}
\newcommand{\myvardot}[1]{\dot{\myvar{#1}}}
\newcommand{\apvar}[1]{\bmh{#1}}
\newcommand{\apvardot}[1]{\bmdh{#1}}
\newcommand{\apmat}[1]{\hat{#1}}

\begin{document}

\title{
Robust Object Manipulation for Tactile-based Blind Grasping
}

\author[melb]{Wenceslao Shaw-Cortez\corref{cor1}}
\ead{shaww@student.unimelb.edu.au}

\author[melb]{Denny Oetomo \corref{cor2}}
\ead{doetomo@unimelb.edu.au}

\author[melb]{Chris Manzie}
\ead{manziec@unimelb.edu.au}

\author[vin]{Peter Choong}
\ead{pchoong@unimelb.edu.au}

\cortext[cor1]{Corresponding author}

\address[melb]{School of Electrical, Mechanical, and Infrastructure Engineering, The University of Melbourne, 3010, VIC, Australia}

\address[vin]{Department of Surgery, St. Vincent's Hospital, 3065, VIC, Australia}

\begin{abstract}

Tactile-based blind grasping addresses realistic robotic grasping in which the hand only has access to proprioceptive and tactile sensors. The robotic hand has no prior knowledge of the object/grasp properties, such as object weight, inertia, and shape. There exists no manipulation controller that rigorously guarantees object manipulation in such a setting. Here, a robust control law is proposed for object manipulation in tactile-based blind grasping. The analysis ensures semi-global asymptotic and exponential stability in the presence of model uncertainties and external disturbances that are neglected in related work. Simulation and hardware results validate the effectiveness of the proposed approach.

\end{abstract}


\maketitle

\section{Introduction}

Object manipulation via robotic hands is an ability that has been pursued for decades. One specific type of object manipulation is in-hand manipulation in which the object is translated and/or rotated within a grasp, and is only in contact with the fingertips. In-hand manipulation, as opposed to static grasping, requires more precise control of the robotic hand to apply the appropriate contact forces to move the object. In addition to moving the object, the robotic hand is responsible for ensuring the object stays within the grasp without slipping or losing contact with the fingertips. All of this must be accomplished despite the effects of rolling, inertial/Coriolis effects, and external disturbances that interplay the hand and object relationship.  

In addition to the complexities inherent in object manipulation, this work focuses on in-hand manipulation for unstructured environments where the robotic hand is deployed in the real world. As such, the manipulation controller only has access to variables that can be measured or observed by the sensors on the robotic hand. Vision-based sensors, for example, cannot provide the object center of mass, but can track the object relative pose \cite{Garcia2009}. However most vision-based methods, in addition to those that require object pose and velocity, require markers to be placed on the object prior to grasping \cite{Jara2014, Ueki2008}, which makes them highly impractical  in unstructured environments, and restricts them to laboratory settings. Needless to say, humans are able to manipulate objects using only proprioceptive and tactile sensing, and are thus not dependent on vision for manipulation tasks. This motivates the definition of \textit{tactile-based blind grasping} that will be used throughout this work. In tactile-based blind grasping, the robotic hand only has access to proprioceptive and tactile sensors \cite{Cutkosky2016, Kappassov2015} that can be physically integrated into the robotic hand for use in unstructured environments. 

Although tactile information can be used to obtain valuable grasp information, there is still significant object model uncertainty regarding the object center of mass, shape, inertia, weight, friction, etc. When robots are deployed in the real world, it is unreasonable to assume intimate knowledge of every object that needs to be grasped, including its shape, inertia, center of mass, and/or possible wrench disturbances that can act on it.  Additionally, uncertainties in the hand model, including kinematic and dynamic models of the hand, are always present in practice as the model never truly matches the physical system. Thus in-hand manipulation for tactile-based blind grasping is synonymous with robust in-hand manipulation that must account for rolling, inertial/Coriolis effects, and external disturbances, while handling uncertainties in the hand-object model and sensor measurements.

\subsection*{Related Work}

In-hand manipulation requires an intial grasp to be made a priori. Existing work has provided robust means of controlling end effector positions and performing grasp synthesis \cite{Zanchettin2017,Hang2016,Lippiello2013a}. Thus the focus here is on in-hand manipulation control once a grasp has been formed. Initial work for in-hand manipulation proposed methodologies for modeling the hand-object system, as well as analyzing properties of the grasp \cite{Kerr1986, Cole1989, Cutkosky1989, Murray1994}. However due to the complexity in controlling robotic hands for in-hand manipulation,  different types of solutions have emerged. Some existing solutions include grasp force optimization \cite{Bicchi2000a,Buss1996,Fungtammasan2012} and motion planning \cite{ Cherif1998, Saut2011,Hertkorn2013}. Those approaches typically assume properties such as object weight, contact friction, and center of mass are known a priori and that the hand-object system is quasi-static. Despite the novelty in those solutions, their inherent assumptions are overly constraining for tactile-based blind grasping. Furthermore, the quasi-static assumption used in much of the literature ignores the dynamics of the system, which is critical for manipulation tasks. That assumption simplifies the problem, but at the expense of not guaranteeing stability of the hand-object system.  

Most early manipulation controllers were heavily reliant on the hand-object model. Computed torque methods \cite{Cole1989} and feedback linearization \cite{Sarkar1997} have been applied to object manipulation, but those methods required exact knowledge of the object model and state, and were not robust to model uncertainties that occur in practice. One exception reported in \cite{Jen1996} presents a sliding mode controller in which the control input is the contact force, however that control did not consider finger nor contact dynamics/kinematics to apply the necessary contact forces.

Adaptive control methods have also been developed to handle model uncertainties in grasping/manipulation. In \cite{Cheah1998} an adaptive PD control law was proposed to compensate for uncertain gravity, contact locations, and hand kinematics. A trajectory tracking adaptive control law was presented in \cite{Ueki2008} that guarantees asymptotic stability. That controller was then extended to consider compressible fingertips, and a robust/adaptive controller was proposed, which guarantees uniformly ultimate bounded tracking error \cite{Ueki2011}. Although those adaptive control methods provide robustness to model uncertainty, they all require measurements of the object pose that are not available in tactile-based blind grasping. Other robust methods also require object pose measurements \cite{Jara2014},  neglect rolling \cite{Fan2017,Li2010}, and external disturbances \cite{Caldas2015} which makes them unsuitable for tactile-based blind grasping. 

Manipulation controllers were also developed using passivity-based analysis. Initial work in passivity-based analysis for object manipulation was restricted to planar grasping/manipulation \cite{Arimoto2000, Doulgeri2002, Ozawa2004, Arimoto2008,Grammatikopoulou2014}, which was extended in \cite{Arimoto2006} to cuboid objects, but only for grasp stability. A passive, spatial manipulation controller was then developed in \cite{Wimboeck2006}, which was referred to as the intrinsically passive controller. The authors defined a virtual frame to replace the model of the object. That control was then extended to consider internal dynamics \cite{Wimbock2008}, and a comparison of various types of passivity-based controllers was presented in \cite{Wimbock2012}. The notion of the virtual frame was extended in \cite{Tahara2010} into a stabilizing control law for cuboid shaped objects that only required joint angle and joint velocity measurements, which was referred to as the blind grasping controller.  That controller was extended to more general polyhedral objects in \cite{Kawamura2013}. However those passivity-based methods require conservative assumptions that the object is weightless and polyhedral \cite{Kawamura2013, Tahara2010}, or neglect rolling effects and elements of the hand-object dynamics  \cite{Wimbock2012,Wimbock2008}. Furthermore those approaches do not consider consider hand kinematic uncertainties that result from implementation.

Although the literature has progressed to develop manipulation controllers that do not require object state information, a truly robust control method has yet to be developed for tactile-based blind grasping. The early control techniques required exact models of the hand-object system, with no consideration of model uncertainties. The more robust methods, such as adaptive control, relied on object pose and velocity measurements that are not available in tactile-based blind grasping. All the passivity-related control methods for spatial manipulation assume an exact model of the hand kinematics and ignore elements of the system dynamics to achieve stability results. Some are restricted to only polyhedral objects. In practice, objects vary in mass/inertia, shape, surface curvature, and all objects have mass. The system dynamics including rolling, Coriolis terms, and gravity/external disturbances are inherent in the dynamics of object manipulation, and cannot be ignored. To date, there exists no control method that guarantees asymptotic stability for in-hand manipulation that addresses rolling, external disturbances, Coriolis terms, and uncertainties in the hand-object model for tactile-based blind grasping. 

This paper extends upon \cite{ShawCortez2016, ShawCortez2017} to develop a robust in-hand manipulation controller  for tactile-based blind grasping. The main contribution of this paper is the development, and associated stability guarantees, of an in-hand manipulation controller that is robust to the model uncertainties which include object mass/inertia, shape, Coriolis terms, hand/object kinematics, and external wrenches all subject to the effects of rolling contacts. The results presented here provide semi-global asymptotic/exponential stability of the system with appropriate tuning guidelines. The proposed control is also extended to vision-based methods and existing disturbance compensators for improved manipulation performance and wider applicability of the proposed approach. Note the stability guarantees presented here satisfy the stability-related assumptions made in related work \cite{ShawCortez2018b}, which then addresses slip and sampling time effects.

\subsection*{Notation}

Throughout this paper, an indexed vector $\myvar{v}_i \in \mathbb{R}^p$ has an associated concatenated vector $\myvar{v} \in \mathbb{R}^{pn}$, where the index $i$ is specifically used to index over the $n$ contact points in the grasp. The notation $\myvar{v}^{\mathcal{\mathcal{E}}}$ indicates that the vector $\myvar{v}$ is written with respect to a frame $\mathcal{E}$, and if there is no explicit frame defined, $\myvar{v}$ is written with respect to the inertial frame, $\mathcal{P}$. The operator $(\cdot)\times$ denotes the skew-symmetric matrix representation of the cross-product. $SO(3)$ denotes the special orthogonal group of dimension 3. The approximation of $\myvar{v}$ is denoted $\apvar{v}$. The minimum and maximum eigenvalues of a positive-definite matrix, $B$, are respectively denoted by $\lambda_{\text{min}}(B)$, and $\lambda_{\text{max}}(B)$. The kernel or null-space of a matrix, $B$, is denoted by $\text{Ker}(B)$. The Moore-Penrose generalized inverse of $B$ is denoted $B^\dagger$. The $n\times n$ identity matrix is denoted $I_{n\times n}$.

\section{System Model} \label{sec:system model}

\subsection{Hand-Object System}

Consider a fully-actuated, multi-fingered hand grasping a rigid, convex object at $n$ contact points.  Each finger consists of $m_i$ joints with smooth, convex fingertips of high stiffness. Let the finger joint configuration be described by the joint angles, $\myvar{q}_i \in \mathbb{R}^{m_i}$. The full hand configuration is defined by the joint angle vector, $\myvar{q} = (\myvar{q}_1, \myvar{q}_2, ..., \myvar{q}_n)^T \in \mathbb{R}^m$, where $m = \sum_{i=1}^n m_i$ is the total number of joints. Let the inertial frame, $\mathcal{P}$, be fixed on the palm of the hand, and a fingertip frame, $\mathcal{F}_i$, fixed at the point $\myvar{p}_{f_i} \in \mathbb{R}^3$. The contact frame, $\mathcal{C}_i$, is located at the contact point, $\myvar{p}_{c_i} \in \mathbb{R}^3$.  A visual representation of the contact geometry for the $i$th finger is shown in Figure \ref{fig.contactpic}. Note a fixed point on the fingertip surface is defined by $\myvar{p}_{ft_i} \in \mathbb{R}^3$, which is fixed with respect to $\mathcal{F}_i$. The inertial position of this fixed point is $\myvar{p}_{t_i} = \myvar{p}_{f_i} + \myvar{p}_{ft_i}$.

\begin{figure}[hbtp]
\centering
\includegraphics[scale=0.37]{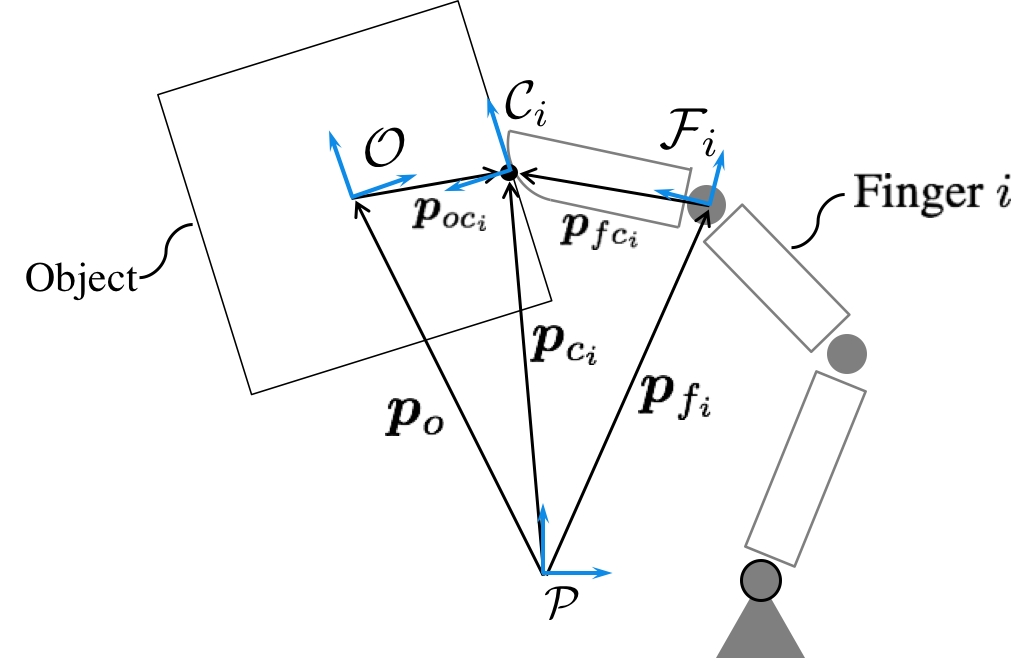}
\caption{A visual representation of the contact geometry for contact $i$. }  \label{fig.contactpic}
\end{figure}

The hand Jacobian, $J_h := J_h(\myvar{q},\myvar{p}_{fc}) \in \mathbb{R}^{3n \times m}$ defines the kinematics of the hand. The full Jacobian, $J_h$ is constructed by combining each $J_{h_i}(\myvar{q}_i,\myvar{p}_{fc_i})\in \mathbb{R}^{3 \times m_i} $ into a block diagonal matrix \cite{Cole1989}:
\begin{equation} \label{eq:hand Jacobian}
J_{h_i}(\myvar{q}_i,\myvar{p}_{fc_i}) = \left[ \begin{array}{c c} I_{3\times 3} & -(\myvar{p}_{fc_i})\times \end{array} \right] J_{s_i}(\myvar{q}_i)
\end{equation}
where $\myvar{p}_{fc_i} \in \mathbb{R}^3$ is the vector from $\mathcal{F}_i$ to $\mathcal{C}_i$, and $J_{s_i}(\myvar{q}_i) \in \mathbb{R}^{6 \times m_i}$ is the manipulator Jacobian relating $\myvardot{q}_i$ with the translational and rotational velocities about $\myvar{p}_{f_i}$ (see Figure \ref{fig.contactpic}). 

Let $\mathcal{O}$ be a reference frame fixed at the object center of mass $\myvar{p}_o \in \mathbb{R}^3$, and $R_{po} \in  SO(3)$ is the rotation matrix, which maps from $\mathcal{O}$ to $\mathcal{P}$. The angular velocity of the object frame with respect to $\mathcal{P}$ is $\myvar{\omega}_o \in \mathbb{R}^3$. The object pose is defined by $\myvar{x}_o \in \mathbb{R}^6$, with $\myvardot{x}_o = (\myvardot{p}_o, \myvar{\omega}_o)$.  The position vector from the object center of mass to the respective contact point is $\myvar{p}_{oc_i} \in \mathbb{R}^3$ (see Figure \ref{fig.contactpic}) . 

Each fingertip exerts a contact force, $\myvar{f}_{c_i} \in \mathbb{R}^3$, on the object at the contact point, $\myvar{p}_{c_i} \in \mathbb{R}^3$. Let the matrix $G_i(\myvar{p}_{oc_i}) \in \mathbf{R}^{6 \times 3}$ be the map from the contact force, $\myvar{f}_{c_i}$, to the corresponding wrench acting on the object. The transpose, $G_i(\myvar{p}_{oc_i})^T$, maps the object motion to the velocity of the $i$th contact point. Using a point contact with friction model, $G_i(\myvar{p}_{oc_i})^T$ can be computed by \cite{Cole1989}:
\begin{equation}  \label{eq:grasp map}
G_i^T(\myvar{p}_{oc_i})  = \left[ \begin{array}{c c} I_{3\times 3} & -(\myvar{p}_{oc_i})\times \end{array} \right] 
\end{equation}
The grasp map, $G := G(\myvar{p}_{oc}) \in \mathbb{R}^{6 \times 3n}$ maps the contact force vector, $\myvar{f}_c$, to the net object wrench, and is defined by:
\begin{equation}\label{eq:full grasp map}
 G(\myvar{p}_{oc})= [ G_1, G_2, ..., G_n]
 \end{equation} 
 
 The hand and object kinematics are related by the following nonholonomic grasp constraint \cite{Murray1994}:
\begin{equation} \label{eq:grasp constraint}
J_h \dot{\myvar{q}} = G^T \myvardot{x}_o
\end{equation}

The following assumptions are made for the hand and object:

\begin{assumption}\label{asm:full rank Jh}
The hand has $m = 3n$ joints, and never reaches a singular configuration.
\end{assumption}

\begin{remark}\label{rm:full rank Jh}
Assumption \ref{asm:full rank Jh} ensures $J_h$ is square and invertible, which is a common assumption in related work \cite{Caldas2015, Wimbock2012}. This assumption is made in order to not distract from the main contribution of the paper and can be relaxed by considering internal motion as will be discussed later.
\end{remark}

\begin{assumption}\label{asm:full rank G}
The multi-fingered grasp has $n>2$ contact points, which are non-collinear.
\end{assumption}

\begin{remark} \label{rm:full rank G}
Assumption \ref{asm:full rank G} ensures $G$ is full rank \cite{Cole1989}. Note, this requires that a grasp is already formed for a manipulation task. The motivation behind such a requirement is to propose a low-level control framework as part of a hierarchical grasping architecture such as \cite{Hang2016,Lippiello2013a}.
\end{remark}

\begin{assumption}\label{asm:slip prevention}
The fingertips can roll on the contact surface with the object, but do not slip or lose contact.
\end{assumption}

\begin{remark}\label{rm:slip prevention}
Assumption \ref{asm:slip prevention} ensures \eqref{eq:grasp constraint} is always satisfied.
\end{remark}

\begin{assumption}\label{asm:smooth surfaces}
The fingertip and object surfaces at the contact points are locally smooth such that the system dynamics are smooth.
\end{assumption}

Under Assumptions \ref{asm:full rank Jh} and \ref{asm:slip prevention}, the hand-object dynamics can be derived as in \cite{Murray1994}:
\begin{equation} \label{eq:handobject system}
M_{ho}\bmdd{x}_o + C_{ho} \myvardot{x}_o = GJ_h^{-T}\Big( \myvar{u}+\myvar{\tau}_e\Big) + \myvar{w}_e 
\end{equation}
with 
\begin{equation} \label{eq:Mho}
M_{ho} = M_o + GJ_h^{-T} M_h J_h^{-1}G^T,
\end{equation}
\begin{equation} \label{eq:Cho}
C_{ho} = C_o + GJ_h^{-T}\Big( C_h J_h^{-1}G^T \\
+  M_h \frac{d}{dt}[ J_h^{-1}G^T] \Big),
\end{equation}
where $M_h:= M_h(\myvar{q}) \in \mathbb{R}^{m\times m}, M_o:= M_o(\myvar{x}_o) \in \mathbb{R}^{6\times 6}$ are respectively the hand and object inertia matrices, $C_h:= C_h(\myvar{q},\myvardot{q}) \in \mathbb{R}^{m\times m}, C_o:= C_o(\myvar{x}_o,\myvardot{x}_o) \in \mathbb{R}^{6\times 6}$ are the respective hand and object Coriolis matrices, $\myvar{\tau}_e := \myvar{\tau}_e(t,\myvar{q},\myvardot{q}) \in \mathbb{R}^m$ is the sum of all dissipative and non-dissipative disturbance torques acting on the hand,  $\myvar{w}_e := \myvar{w}_e(t) \in \mathbb{R}^6$ is an external wrench disturbing the object, and  $\myvar{u} \in \mathbb{R}^m$ is the joint torque control input for a fully actuated hand. The hand-object inertia and Coriolis matrices are denoted by $M_{ho} := M_{ho}(\myvar{q},\myvar{x}_o) \in \mathbb{R}^{6\times 6}$ and $C_{ho} := C_{ho}(\myvar{q}, \myvar{x}_o, \myvardot{q}, \myvardot{x}_o) \in \mathbb{R}^{6\times 6}$, respectively. The following assumption is made for the external wrenches that may act on the system:

\begin{assumption}\label{asm:constant disturbance}
The disturbance terms, $\myvar{\tau}_e, \myvar{w}_e$ are bounded, continuously differentiable, and are constant when the hand-object system is at rest such that: 
\begin{equation}
(\myvardot{x}_o,\myvardot{q}) \equiv 0 \implies \myvardot{\tau}_e, \myvardot{w}_e = 0
\end{equation}
\end{assumption}
\begin{remark}\label{rm:constant disturbance}
Common disturbances that satisfy Assumption \ref{asm:constant disturbance} include gravity acting on both the hand and object, and viscous friction acting on the joints in the form of $-\beta \myvardot{q}$ for $\beta \in \mathbb{R}_{>0}$. 
\end{remark} 

\begin{remark}
It is important to note that for rolling contacts, $\myvar{p}_c$ is a function of the hand configuration, object configuration, and geometry of the object and fingertip surfaces. For smooth, convex surfaces there exists a smooth local bijection between the geometry of the fingertip/object surfaces and the hand-object configurations such that $\myvar{p}_c$, $J_h$ and $G$ can be expressed as functions of the hand-object state, $(\myvar{q},\myvar{x}_o)$ \cite{Murray1994}. 
\end{remark}

The following lemma ensures well-known properties of $M_{ho}$:
\begin{lemma}\label{lemma:Mho uniformly bounded}
\cite{Spong1989,Murray1994} Under Assumptions \ref{asm:full rank Jh} and \ref{asm:full rank G}, $M_{ho}$ is positive definite, uniformly bounded such that there exist constants $m_{\text{min}}, m_{\text{max}} \in \mathbb{R}$ that satisfy:
\begin{equation}\label{eq:Mho uniformly bounded}
0 < m_{\text{min}} \leq ||M_{ho}^{-1} || \leq m_{\text{max}}
\end{equation}

\end{lemma}

The hand-object dynamics of \eqref{eq:handobject system}, as well as Assumptions \ref{asm:full rank Jh}-\ref{asm:smooth surfaces} are standard in the manipulation literature \cite{Ueki2008, Cole1989, Murray1994, Fan2017, Wimbock2012}.  Existing controllers provide object manipulation capabilities when the entire hand-object model/states are known. However in tactile-based blind grasping the lack of object information results in unknown/uncertain $M_{ho}$, $C_{ho}$, $G$, $\myvar{\tau}_e$, $\myvar{w}_e$, and even $J_h$. Furthermore the object state $\myvar{x}_o$ is not available to the on-board controller. Amongst the tactile-based blind grasping literature, Assumption \ref{asm:constant disturbance} is less conservative than existing methods which neglect external disturbances altogether \cite{Tahara2010, Kawamura2013} or neglect rolling and $C_{ho}$ terms \cite{Wimbock2008,Wimbock2012}. This presents the problem of robust object manipulation addressed in this work and formally presented in the following section.

\subsection{Problem Formulation}\label{ssec:taskframesystem}

\subsubsection*{Task Frame Definition}

In the related work, the notion of the ``virtual frame"  \cite{Wimbock2012, Tahara2010} has been used, and has been accepted for manipulation tasks \cite{Hang2016}. The idea of the ``virtual frame" is to define a pseudo object whose state is only dependent on the hand configuration. That is, the center of the pseudo object, $\myvar{p}_a \in \mathbb{R}^3$, is the centroid of the grasp fingertips, and the orientation of the pseudo object is defined by the positions of the fingertips, and represented by the rotation matrix $R_{pa} \in SO(3)$.  Thus the virtual frame is only a function of the joint angles, and is defined by:
\begin{equation}\label{eq:define pa for blind grasping}
\myvar{p}_a(\myvar{q}) = \frac{1}{n}\sum_{i=0}^n \myvar{p}_{t_i}(\myvar{q}_i)
\end{equation} 
 \begin{equation}\label{eq:define Rpa for blind grasping}
 R_{pa}(\myvar{q}) = [\myvar{\rho}_x, \myvar{\rho}_y, \myvar{\rho}_z]
 \end{equation}
 where $\myvar{\rho}_x = \myvar{\rho}_y \times \myvar{\rho}_z,
\myvar{\rho}_y = \frac{ \myvar{p}_{t_1} - \myvar{p}_{t_2} }{ ||\myvar{p}_{t_1} - \myvar{p}_{t_2} ||_2 },
\myvar{\rho}_z = \frac{ (\myvar{p}_{t_3}-\myvar{p}_{t_1})\times (\myvar{p}_{t_2}-\myvar{p}_{t_1}) }{ ||  (\myvar{p}_{t_3}-\myvar{p}_{t_1}) \times (\myvar{p}_{t_2}-\myvar{p}_{t_1}) ||_2 }$. 

The virtual frame is used here to define the task frame for a desired manipulation motion without requiring knowledge of the object model. Let $\mathcal{A}$ be the task frame located at the point $\myvar{p}_a$ with respect to $\mathcal{P}$, with rotation matrix $R_{pa}$ that maps $\mathcal{A}$ to $\mathcal{P}$.  Let  $\myvar{v}_a \in \mathbb{R}^3$ denote the velocity of $\myvar{p}_a$, and $\myvar{\omega}_a \in \mathbb{R}^3$ denote the angular velocity of frame $\mathcal{A}$ with respect to $\mathcal{P}$. 

For practical considerations, a local parameterization of $SO(3)$ is used to define a notion of orientation error by defining $\myvar{\gamma}_a \in \mathbb{R}^3$, such that $R_{pa} = R_{pa}(\myvar{\gamma}_a)$ \cite{Cole1989}. One example of such a local parameterization is:
\begin{equation}\label{eq:define ga for blind grasping}
\myvar{\gamma}_a(\myvar{q}) = \left[ \begin{array}{c}\arctan( -\myvar{\rho}_{z_2}/ \myvar{\rho}_{z_3} )\\ \sqrt{1-\myvar{\rho}_{z_1}^2} \\ \arctan(-\myvar{\rho}_{y_1}/ \myvar{\rho}_{x_1}) \end{array} \right]
\end{equation}
To incorporate this local parameterization in the kinematics, let $S(\myvar{\gamma}_a) \in \mathbb{R}^{3\times 3}$ denote the one-to-one mapping defined by:
\begin{equation}\label{eq:velocity gamma_a}
 \myvar{\omega}_a = S(\myvar{\gamma}_a) \myvardot{\gamma}_a 
\end{equation}
The matrix $S(\myvar{\gamma}_a)$ is absorbed into $P := \text{diag}(I_{3\times 3}, S(\myvar{\gamma}_a))$ such that:
\begin{equation}\label{eq:P def}
\left[ \begin{array}{c} \myvardot{p}_a \\ \myvar{\omega}_a \end{array} \right] = P \myvardot{x}
\end{equation}
It is inherently assumed that the orientation $\myvar{\gamma}_a$ does not pass through a singular configuration. 

The task frame state is $\myvar{x} = (\myvar{p}_a, \myvar{\gamma}_a) \in \mathbb{R}^6$. Finally, let $\frac{\partial \myvar{x}}{\partial \myvar{q}} \in \mathbb{R}^{6 \times 3n}$ denote the Jacobian of the task frame that maps $\myvardot{q}$ to $\myvardot{x}$. The following assumption is used in related work \cite{Wimbock2012, Tahara2010}:
\begin{assumption} \label{asm:full rank Jc}
The function $\myvar{x}(\myvar{q})$ is continuously differentiable, and $\frac{\partial \myvar{x}}{\partial \myvar{q}}$ is full rank.
\end{assumption}

\subsubsection*{Control Objective}

For set-point object manipulation, the state $\myvar{x}$ must reach a desired reference $\myvar{r} \in \mathbb{R}^6$, where $\myvardot{r}, \bmdd{r} \equiv 0$. Let $\myvar{e} = \myvar{x} - \myvar{r}$ define the error. The objective of the proposed control algorithm is to asymptotically reach $(\myvar{e},\myvardot{e}) = 0$ in the presence of uncertain disturbances. The control problem is defined as follows:
\vspace{1.5mm}
\begin{problem}\label{pr:disturbance rejection problem}
Given a hand-object system that satisfies Assumptions \ref{asm:full rank Jh}-\ref{asm:full rank Jc}, determine a control law that semi-globally satisfies:
\begin{equation} \label{eq:equilibrium condition}
 \underset{t \to \infty}{\text{lim}} (\myvar{e}(t),\myvardot{e}(t)) \to 0
\end{equation}
\end{problem}


\section{Proposed Solution}\label{sec:proposed solution}

\subsection{Proposed Controller}

The proposed control is defined as:
\begin{equation}\label{eq:proposed control}
\myvar{u} = \apmat{J}_h^T \left( (P^T \apmat{G}) ^\dagger  \myvar{u}_m +  \myvar{u}_f \right)
\end{equation}
where $\apmat{J}_h^T$ and $\apmat{G}$ are full rank approximations of $J_h$ and $G$ respectively, and $\myvar{u}_m \in \mathbb{R}^m$ is the PID-based manipulation controller:
\begin{equation}\label{eq:PID control}
\myvar{u}_m =  -K_p \myvar{e} - K_i \int_0^t \myvar{e} \ dt - K_d \myvardot{e}
\end{equation}
where $K_p,K_i, K_d \in \mathbb{R}^{6\times 6}$ are the respective proportional, integral, and derivative positive-definite gain matrices. 

The approximations $\apmat{J}_h$ and $\apmat{G}$  are solely defined as functions of the joint angles $\myvar{q}$. The approximate hand Jacobian, $\apmat{J}_h$, is a block diagonal matrix composed of $\apmat{J}_{h_i}$ for each $i$ finger:
\begin{equation}\label{eq:Jh uncertain contact and model}
\apmat{J}_{h_i}(\myvar{q}_i) = \left[ \begin{array}{c c} I_{3\times 3} & -(\myvar{p}_{ft_i})\times \end{array} \right] \apmat{J}_{s_i}(\myvar{q}_i)
\end{equation}
where $\apmat{J}_{s_i}$ refers to the approximate spatial Jacobian resulting from approximations in the link lengths and joint positions. The approximation $\apmat{G}$ is defined as:
 \begin{equation}\label{eq:uncertain G vision}
\apmat{G}(\myvar{q}) = \left[ \begin{array}{cccc} I_{3\times 3}, & ... ,& I_{3\times 3} ,\\ (\myvar{p}_{t_1}-\myvar{p}_a)\times, & ... ,& (\myvar{p}_{t_n}-\myvar{p}_a)\times \end{array} \right]
\end{equation}
It is important to emphasize that $\myvar{p}_{ft_i}$ is fixed with respect to the fingertip surface, such that $\apmat{J}_{h_i}$ is only a function of $\myvar{q}$ and is not equal to $J_h$. Similarly, $\myvar{p}_{t_i}$ and $\myvar{p}_a$ are solely dependent on $\myvar{q}$, such that $\apmat{G}$ is an erroneous approximation of $G$.  Thus the proposed control, apart from $\myvar{u}_f$, only requires the joint angles/velocities $\myvar{q}, \myvardot{q}$ and reference $\myvar{r}$. Also, the only model required by the control is the approximate hand kinematics $\apmat{J}_{s_i}(\myvar{q})$. Note however that the discrepancies between $J_h$, $\apmat{J}_h$, and $G$, $\apmat{G}$ are not neglected, and are addressed in the stability analysis presented in the following section.

The internal force control input $\myvar{u}_f \in \mathbb{R}^{3n}$ is used to control the internal forces of the grasp, and prevent slip such that Assumption \ref{asm:slip prevention} holds. One common property in related work of the internal force controller is that it is constant when the hand-object system is static \cite{Kerr1986,Bicchi2000a, Buss1996, Fungtammasan2012}. This is formalized in the following assumption:

\begin{assumption}\label{asm:internal force condition}
The internal force control satisfies:
\begin{equation}\label{eq:uf condition}
\myvardot{x} \equiv 0 \implies \myvardot{u}_f = 0
\end{equation}
\end{assumption}

The choice of an internal force controller is not unique. One possible solution for $\myvar{u}_f$ is the internal force control law:
\begin{equation}\label{eq:centroid uf}
\myvar{u}_f(\myvar{q}) = k_f (\myvar{p}_a - \myvar{p}_{t_1}, \myvar{p}_a - \myvar{p}_{t_2}, ..., \myvar{p}_a - \myvar{p}_{t_n})^T
\end{equation}
where $k_f \in \mathbb{R}_{>0}$ is a scalar gain \cite{Kawamura2013, Wimbock2012, Bae2012}. Note that if \eqref{eq:centroid uf} is used to define $\bm{u}_f$, then the proposed control \eqref{eq:proposed control}, \eqref{eq:PID control} is a \textit{blind grasping} controller in that it only requires proprioceptive measurements of $\bm{q}, \bmd{q}$ in addition to the reference $\myvar{r}$. However, the internal force defined by \eqref{eq:centroid uf} does not guarantee Assumption \ref{asm:slip prevention} holds generally.

A systematic way of defining $\bm{u}_f$ to ensure Assumption \ref{asm:slip prevention} (no slip) holds is via a technique known as grasp force optimization \cite{Bicchi2000a, Buss1996, Fungtammasan2012}. In grasp force optimization, the generalized inverse from \eqref{eq:proposed control} is re-written as a quadratic program:
\begin{equation}\label{eq:proposed control in gfo}
\myvar{u} = \apmat{J}_h^T \myvar{u}_{fc}^*,
\end{equation}
\begin{align} \label{eq:gfo control}
\begin{split}
\myvar{u}_{fc}^* \hspace{0.1cm} = \hspace{0.1cm} & \underset{\myvar{v}}{\text{argmin}}
\hspace{.3cm} \myvar{v}^T\myvar{v} \\
& \hspace{.01cm} \text{s.t.} 
\hspace{0.4cm} \apmat{G} \myvar{v} = \myvar{u}_m \\ 
\end{split}
\end{align}
Then, linearized friction cone constraints are incorporated into \eqref{eq:gfo control} to ensure Assumption \ref{asm:slip prevention} is satisfied. The cost is usually used to minimize internal forces or actuator effort. Thus in grasp force optimization, the internal force controller is incorprated in the optimization to deal with these additional constraints and grasp redundancies. However \eqref{eq:proposed control in gfo} typically requires exact knowledge of the object model. Grasp force optimization has since been extended to tactile-based blind grasping in \cite{ShawCortez2018b}. However \cite{ShawCortez2018b} is dependent on unfounded stability guarantees associated with the proposed control \eqref{eq:proposed control}, \eqref{eq:PID control}. In the following section, these stability guarantees are formally ensured.

\subsection{Stability Analysis}\label{ssec:proposed control}

The stability analysis presented here is used to identify the disturbances that arise from model uncertainties in tactile-based blind grasping, and then ensure the proposed control law is robust to such disturbances. Using knowledge of the disturbance dynamics along with structural properties of the hand-object system, the analysis shows that the proposed control can achieve semi-global asymptotic and semi-global exponential stability about the origin. 

To start, the system dynamics are derived for the state $\myvar{x}$, by deriving the relation between $\myvardot{x}$ and $\myvardot{x}_o$. From Assumptions \ref{asm:full rank Jh} and \ref{asm:slip prevention}, $\myvardot{q}$ is solved for from \eqref{eq:grasp constraint} and substituted into $\myvardot{x} = \frac{\partial \myvar{x}}{\partial \myvar{q}} \myvardot{q}$:
\begin{equation}\label{eq:relate x to xo}
\myvardot{x} =  \frac{\partial \myvar{x}}{\partial \myvar{q}} J_h^{-1} G^T \myvardot{x}_o
\end{equation} 
For ease of notation, let $J_a = \frac{\partial \myvar{x}}{\partial \myvar{q}} J_h^{-1} G^T$. Note that from Assumptions \ref{asm:full rank Jh}, \ref{asm:full rank G}, and \ref{asm:full rank Jc}, $J_a$ is square and invertible such that $\myvardot{x}_o = J_a^{-1} \myvardot{x}$. The following lemma shows that when $\myvardot{x} = 0$, then the hand-object system is at rest:

\begin{lemma}\label{lemma:no zero dynamics}
Consider the system state, $\myvar{x} = (\myvar{p}_a, \myvar{\gamma}_a)$, with $\myvar{p}_a$ and $\myvar{\gamma}_a$ respectively defined by \eqref{eq:define pa for blind grasping}, \eqref{eq:define ga for blind grasping}. Under Assumptions \ref{asm:full rank Jh}, \ref{asm:full rank G}, \ref{asm:slip prevention}, and \ref{asm:full rank Jc}, if $\myvardot{x} = 0$, then $(\myvardot{q}, \myvardot{x}_o) = 0$. 
\end{lemma}
\begin{proof}
From Assumption \ref{asm:slip prevention}, the relation between $\myvardot{x}$ and $\myvardot{x}_o$ follows from \eqref{eq:relate x to xo}, and that of $\myvardot{x}$ and $\myvardot{q}$ follows from $\myvardot{x} = \frac{\partial \myvar{x}}{\partial \myvar{q}} \myvardot{q}$. Due to the full rank conditions from Assumptions \ref{asm:full rank Jh}, \ref{asm:full rank G}, and \ref{asm:full rank Jc}, it directly follows that $\myvardot{x} =0 \implies (\myvardot{q}, \myvardot{x}_o) = 0$. 
\end{proof}

To derive the dynamics for $\bm{x}$, $\myvardot{x}_o = J_a^{-1} \myvardot{x}$ is differentiated as follows:
\begin{equation}\label{eq:xddot to xaddot}
\bmdd{x}_o = \frac{d}{dt}[J_a^{-1}] \myvardot{x} + J_a^{-1} \bmdd{x}
\end{equation} 

Substitution of \eqref{eq:xddot to xaddot} into \eqref{eq:handobject system}, left multiplication by $J_a^{-T}$, and trivial change of variables from $\myvar{x}$ to $\myvar{e}$ results in the following system dynamics of similar form to \cite{Wimbock2012}:
\begin{equation} \label{eq:hand contact system}
M_a \bmdd{e} + C_a \myvardot{e} =   J_a ^{-T} G J_h^{-T} ( \myvar{u} +\myvar{\tau}_e) +  J_a^{-T}\myvar{w}_e 
\end{equation}
where
\begin{equation}\label{eq:Ma}
M_a = J_a^{-T} M_{ho} J_a^{-1}
\end{equation}
\begin{equation}
C_a = J_a ^{-T} M_{ho} \frac{d}{dt}[J_a^{-1}]  + J_a^{-T} C_{ho} J_a^{-1}
\end{equation}

It is straightforward to see that the inertia matrix $M_a$ is positive definite and ultimately bounded due to the original properties of $M_{ho}$ and the full rank conditions of $J_a$, $J_h$, and $G$. This is summarized in the following lemma:
\begin{lemma}\label{lemma:Ma}
Under Assumptions \ref{asm:full rank Jh}, \ref{asm:full rank G}, and \ref{asm:full rank Jc}, $M_a$ defined by \eqref{eq:Ma} is positive definite, and uniformly bounded.
\end{lemma}

Next, the proposed control law \eqref{eq:proposed control} is substituted into the system dynamics \eqref{eq:hand contact system}:
\begin{equation}\label{eq:hand contact system ver 2}
M_a \bmdd{e} + C_a \myvardot{e} =   J_a ^{-T} G J_h^{-T} \apmat{J}_h^T \left( \apmat{G}^\dagger P^{-T} \myvar{u}_m + \myvar{u}_f \right) + J_a ^{-T} G J_h^{-T} \myvar{\tau}_e+  J_a^{-T}\myvar{w}_e
\end{equation}

It is important to note that the system dynamics \eqref{eq:hand contact system ver 2} are not dependent on assumptions that neglect rolling, Coriolis terms, or external disturbances, which are required in related work \cite{Wimbock2012, Kawamura2013}. Furthermore, from \eqref{eq:hand contact system ver 2}, it is clear that the model uncertainties in the hand and object kinematics contribute to additional disturbances. 

To proceed, the \textit{approximate} matrices from \eqref{eq:proposed control} are multiplied out with their true inverses as shown in \eqref{eq:hand contact system ver 2}. Note the matrix $P$ from the control \eqref{eq:proposed control} is used as an approximation of $J_a^{-1}$. The motivation behind this is that $P \myvardot{x}$ defines the translational and angular velocities of the task frame, which ideally should equal the object velocity, $\myvardot{x}_o$, such that $P \myvardot{x} \approx \myvardot{x}_o = J_a^{-1} \myvardot{x}$. This relation is only an approximation due to the effects of rolling between the object and fingertips. The result of the product of matrices and approximate inverses is as follows:
\begin{equation}\label{eq:hand contact system ver 2.5}
M_a \bmdd{e} + C_a \myvardot{e} =  \myvar{u}_m + J_a ^{-T} G J_h^{-T} \myvar{\tau}_e+  J_a^{-T}\myvar{w}_e \\
+ D_1 \myvar{u}_m + D_2 \myvar{u}_f 
\end{equation}
where $D_1 := D_1(\myvar{x},\myvar{q},\myvar{x}_o) \in \mathbb{R}^{6\times 6}$, $D_2 := D_2(\myvar{x},\myvar{q},\myvar{x}_o) \in \mathbb{R}^{6\times 3n}$ represent the residual matrices that arise from the approximations of $J_h$ and $G$ multiplying their respective true inverses.

The previous derivation of the system dynamics, along with the proposed control, shows that object manipulation for tactile-based blind grasping resembles an Euler-Lagrange system with disturbance terms. To address robustness to these disturbances, it must first be shown that the origin is an equilibrium point of the system. This satisfies the conditions for related work  \cite{Alvarez-Ramirez2000} where singular perturbation analysis is used to guarantee semi-global stability properties. 

In the following derivation, the system dynamics are re-written in a similar singular perturbation structure and show that the origin is the equilibrium point of the system, and thus the desired stability results follow. Note the following derivation is in-line with the method used in \cite{Alvarez-Ramirez2000}. The novelty of the presentation here is in showing the properties of the disturbance in tactile-based blind grasping, which are notably different than those of \cite{Alvarez-Ramirez2000}.

Let $\myvar{\psi} \in \mathbb{R}^6$ denote the cumulative disturbance including the Coriolis and centrifugal terms:
\begin{equation}\label{eq:define psi}
\myvar{\psi} = -C_a \myvardot{e} + J_a ^{-T} G J_h^{-T} \myvar{\tau}_e+  J_a^{-T}\myvar{w}_e + D_1 \myvar{u}_m + D_2 \myvar{u}_f 
\end{equation}
Furthermore, let $\myvar{\eta} \in \mathbb{R}^6$ denote the full system nonlinearities defined by:
\begin{equation}\label{eq:define eta}
\myvar{\eta} = M_a^{-1} \myvar{\psi} + (M_a^{-1} - \apmat{M}^{-1}) \myvar{u}_m
\end{equation}
where $\apmat{M} \in \mathbb{R}^{6\times 6}$ is a constant, positive definite matrix that approximates $M_a$ and satisfies:
\begin{equation}\label{eq:Mbar condition}
|| I_{6\times 6} - M_a^{-1} \apmat{M} || < 1
\end{equation}
\begin{remark}\label{rm:Mbar condition}
Lemma \ref{lemma:Ma} guarantees the existence of such an $\apmat{M}$ that satisfies \eqref{eq:Mbar condition}. An acceptable choice is $\apmat{M} = \frac{2}{m_{\text{max}}} I_{6\times 6}$ for $m_{\text{max}} \in \mathbb{R}_{>0}$ chosen sufficiently large to exceed the bounded norm on $M_a$. Note a less conservative $\apmat{M}$ can be defined as a diagonal matrix whose elements are upper bounds of the diagonal elements of $M_a$ \cite{Alvarez-Ramirez2000}.
\end{remark}

The system dynamics \eqref{eq:hand contact system ver 2.5} is re-written using \eqref{eq:define psi} and \eqref{eq:define eta}
\begin{equation}\label{eq:hand contact system ver 3}
\bmdd{e} = \apmat{M}^{-1} \myvar{u}_m + \myvar{\eta}
\end{equation}

To introduce the time-scale separation for the desired singularly perturbed structure, the PID controller \eqref{eq:PID control} is re-written as the following equivalent control law as per Proposition 1 of \cite{Alvarez-Ramirez2000}:
\begin{equation}\label{eq:FL PID control}
\myvar{u}_m = \apmat{M} ( - \apvar{\eta} - K_1 \myvar{e} - K_2 \myvardot{e} )
\end{equation}

\begin{equation}\label{eq:define eta estimation}
\apvar{\eta} = \frac{1}{\varepsilon	} ( \myvar{w} + \myvardot{e} )
\end{equation}

\begin{equation}\label{eq:eta estimation dynamics}
\myvardot{w} = - \apmat{M}^{-1} \myvar{u}_m - \frac{1}{\varepsilon} (\myvar{w} + \myvardot{e}), \ \myvar{w}(0) = -\myvardot{e}(0)
\end{equation}
where $K_1, K_2 \in \mathbb{R}^{6\times 6}$ are positive definite matrices, $\varepsilon \in \mathbb{R}_{>0}$ is the singular perturbation parameter, and $\apvar{\eta} \in \mathbb{R}^{6\times 1}$ is an estimation of $\myvar{\eta}$, whose update law is defined by \eqref{eq:define eta estimation}, \eqref{eq:eta estimation dynamics}.

Substitution of \eqref{eq:FL PID control} into \eqref{eq:hand contact system ver 3} results in:
\begin{equation}\label{eq:hand contact system ver 4}
\bmdd{e} = -K_1 \myvar{e} - K_2 \myvardot{e} + (\myvar{\eta} - \apvar{\eta})
\end{equation}

In \eqref{eq:hand contact system ver 4}, the system dynamics are now separated into linear and nonlinear terms. For stability, the nonlinear component (i.e. the error between $\myvar{\eta}$ and $\apvar{\eta}$) must be shown to converge to zero. To do so, a new state is introduced to define this error:  $\myvar{y} = \myvar{\eta} - \apvar{\eta}$. The dynamics of $\myvar{y}$ are derived by differentiation of $\apvar{\eta}$ and $\myvar{\eta}$. Differentiation of \eqref{eq:define eta estimation}  with substitutions from \eqref{eq:eta estimation dynamics} and \eqref{eq:hand contact system ver 3} results in: $\dot{\apvar{\eta}} = \frac{1}{\varepsilon} \myvar{y}$. Differentiation of \eqref{eq:define eta}, whose derivation is omitted for brevity, results in:

\begin{equation}\label{eq:eta differentiation}
\myvardot{\eta} = -\frac{1}{\varepsilon} (M_a^{-1} \apmat{M} - I_{6\times 6}) \myvar{y} +\myvar{\phi}
\end{equation}

\begin{equation}\label{eq:define phi}
\myvar{\phi} = M_a^{-1} \Big( \dot{M}_a ( K_1 \myvar{e} + K_2 \myvardot{e} - \myvar{y} ) - (\apmat{M} - M_a) ( K_1 \myvar{e} + K_2 \myvardot{e} + K_2 \myvar{y}) + \myvardot{\psi} \Big)
\end{equation}
where $ \frac{\partial M_a}{\partial \myvardot{e}} \in \mathbb{R}^{6\times 6}$ denotes the tensor for the partial derivative of $M_a$ with respect to $\myvardot{e}$. 

Finally, the system dynamics for tactile-based blind grasping \eqref{eq:hand contact system} is re-written in the following singularly perturbed form by combining \eqref{eq:hand contact system ver 4} with $\myvardot{y} = \frac{1}{\varepsilon} \myvar{y} + \myvardot{\eta}$ and \eqref{eq:eta differentiation}:
\begin{subequations} \label{eq:hand contact system singularly perturbed}
\begin{equation}\label{subeq:reduced system}
\bracketmat{c}{\myvardot{e} \\ \bmdd{e}} = \bracketmat{cc}{ 0_{6\times 6} & I_{6\times 6} \\ -K_1 & -K2} \bracketmat{c}{\myvar{e} \\ \myvardot{e}} + \bracketmat{c}{0 \\ \myvar{y}}
\end{equation}
\begin{equation}\label{subeq:boundary layer system}
\varepsilon \myvardot{y} = - M_a^{-1} \apmat{M} \myvar{y} + \varepsilon \myvar{\phi}
\end{equation}
\end{subequations}

The following lemma ensures that the system origin, $(\myvar{e}, \myvardot{e}, \myvar{y}) = 0$, is an equilibrium point:

\begin{lemma}\label{lemma:psi is constant at origin}
Consider the system defined by \eqref{eq:hand contact system singularly perturbed}. Under Assumptions \ref{asm:full rank Jh}-\ref{asm:internal force condition}, $\myvardot{\psi} = 0$ at the origin $(\myvar{e}, \myvardot{e}, \myvar{y}) = 0$, such that the origin is an equilibrium point of the system.
\end{lemma}
\begin{proof}
By Assumption \ref{asm:smooth surfaces}, $\myvardot{\psi}$ is well-defined and so differentiation of $\myvar{\psi}$ defined by \eqref{eq:define psi} results in:
\begin{multline}\label{eq:psi dot}
\myvardot{\psi} = -\frac{d}{dt}[ C_a] \myvardot{e} - C_a \bmdd{e} + \frac{d}{dt}[ J_a^{-T} G J_h^{-T} ] \myvar{\tau}_e + J_a^{-T} G J_h^{-T} \myvardot{\tau}_e \\
+  \frac{d}{dt} [ J_a^{-T} ] \myvar{w}_e + J_a^{-T} \myvardot{w}_e + \frac{d}{dt}[D_1] \myvar{u}_m + D_1 \myvardot{u}_m + \frac{d}{dt}[ D_2] \myvar{u}_f + D_2 \myvardot{u}_f
\end{multline}
It is clear that the first term vanishes at the origin. The second term also disappears after substitution of $\bmdd{e}$ from \eqref{subeq:reduced system}. From Lemma \ref{lemma:no zero dynamics} it follows that at the origin, $\myvardot{q} = 0$, and $ \myvardot{x}_o = 0$. Thus from Assumptions \ref{asm:constant disturbance} and \ref{asm:internal force condition}, it follows that the terms containing $\myvardot{\tau}_e$, $\myvardot{w}$, and $\myvardot{u}_f$ vanish at the origin.

For the term $D_1 \myvardot{u}_m$, differentiation of $\myvar{u}_m$ defined by \eqref{eq:PID control} and evaluation at the origin results in $\myvardot{u}_m(0) = - K_d \bmdd{e}$. Substitution of $\bmdd{e}$ from \eqref{subeq:reduced system} results in $\myvardot{u}_m(0) = 0$.

For the remaining terms in \eqref{eq:psi dot}, it is important to note that $J_h, J_a$ and $G$ are functions of $\myvar{x}$, $\myvardot{q}$, and/or $\myvar{x}_o$. The derivative of these terms can be written in the following form for an arbitrary matrix-valued function $B(\myvar{x},\myvar{q}, \myvar{x}_o)$:
\begin{equation}\label{eq:matrix derivative}
\frac{d}{dt}[B(\myvar{x},\myvar{q}, \myvar{x}_o)] = \sum_{j = 1}^6 \frac{\partial B}{ \partial \myvar{x}_j} \myvardot{x}_j + \frac{\partial B}{ \partial \myvar{x}_{o_j}} \myvardot{x}_{o_j} + \sum_{l=1}^{m} \frac{\partial B}{\partial \myvar{q}_l} \myvardot{q}_l 
\end{equation}
Using Lemma \ref{lemma:no zero dynamics} with Assumptions \ref{asm:full rank Jh}, \ref{asm:full rank G}, and \ref{asm:full rank Jc},  $\myvardot{e} = 0 \implies (\myvardot{x}, \myvardot{q}, \myvardot{x}_o) = 0$. Thus all terms in \eqref{eq:matrix derivative} multiplied by $\myvardot{x}, \myvardot{q}, \myvardot{x}_o$ cancel out, and the remaining terms in \eqref{eq:psi dot} vanish at the origin. Thus $\myvardot{\psi} = 0$ at the origin.

Furthermore, by inspection of $\myvar{\phi}$ defined by \eqref{eq:define phi}, it is clear that $\myvar{\phi} = 0$ at the origin. Thus the origin of \eqref{eq:hand contact system singularly perturbed} is an equilibrium point of the system.
\end{proof}

From Lemma \ref{lemma:psi is constant at origin}, the stability results from \cite{Alvarez-Ramirez2000} directly follow. This is summarized in the following theorem, which ensures semi-global asymptotic stability for object manipulation in tactile-based blind grasping:

\begin{theorem}\label{thm:Main case 1} 
Suppose Assumptions \ref{asm:full rank Jh}-\ref{asm:internal force condition} hold for a given grasp. For any $\Delta \in \mathbb{R}_{>0}$ and for all $|| (\myvar{e}(0), \myvardot{e}(0)) ||_2 < \Delta$, there exist positive definite gains $K_p^*, K_i^*, K_d^* \in \mathbb{R}^{6\times 6}$ such that for all $K_p > K_p^*, K_i > K_i^*, K_d > K_d^*$ the system \eqref{eq:hand contact system} with the control law \eqref{eq:proposed control}, \eqref{eq:PID control} is asymptotically stable.
\end{theorem} 

\begin{proof}

Lemma \ref{lemma:Ma} directly satisfies the assumption of Lemma 1 in \cite{Alvarez-Ramirez2000}. Lemma \ref{lemma:psi is constant at origin} satisfies Assumption 2 of \cite{Alvarez-Ramirez2000} such that condition (a) is true, and the origin is an equilibrium point of the closed-loop system \eqref{eq:hand contact system singularly perturbed}. Thus the proof for semi-global asymptotic stability follows from Proposition 2 of \cite{Alvarez-Ramirez2000}.
\end{proof}

Applying stronger conditions on the gain matrices lead to the following result:
\begin{corollary}\label{cor:exponential stability}
Under Assumptions \ref{asm:full rank Jh}-\ref{asm:internal force condition}, there exist positive definite gains $K_p^{**}, K_i^{**}, K_d^{**} \in \mathbb{R}^6$ such that for all $K_p > K_p^{**} \geq K_p^*, K_i > K_i^{**} \geq K_i^*, K_d > K_d^{**} \geq K_d^*$ the system \eqref{eq:hand contact system} with the control law \eqref{eq:proposed control}, \eqref{eq:PID control} is exponentially stable.
\end{corollary} 

\begin{proof}
With Lemmas \ref{lemma:Ma} and \ref{lemma:psi is constant at origin}, the proof for semi-global exponential stability follows directly from Corollary 3 of \cite{Alvarez-Ramirez2000} and Theorem 11.4 of \cite{Khalil2002}.
\end{proof}

\begin{remark}\label{remark:robustness to small perturbations}
Exponential stability provides additional robustness to the system with respect to small perturbations that can relax the constant disturbance condition from Assumption \ref{asm:constant disturbance}. Such perturbations in the grasping scenario may arise from further modeling errors associated with the point contact with friction model, rigid contact surfaces, and bounded external disturbances of small magnitudes that do not vanish at the origin. Note in the case that noise levels are sufficiently high, the use of high gain control to achieve exponential stability may be restricted.
\end{remark}

\begin{remark}
Note that the proposed controller is a blind grasping control law (i.e. where only $\myvar{q}$, $\myvardot{q}$, $\myvar{r}$ are available) when using $\bm{u}_f$ defined by \eqref{eq:centroid uf}, which is also used in related blind grasping research \cite{Tahara2010, Kawamura2013}. However in related work, the control solution causes an induced rolling disturbance, which requires additional control terms for compensation \cite{Kawamura2013}. In the proposed formulation presented here, the effect of no external information manifests as a disturbance which is similarly a function of the proposed manipulation and internal force control terms as seen in \eqref{eq:hand contact system ver 2.5}. However the proposed control neatly compensates for these disturbances without requiring any additional control terms, and furthermore rejects unknown external disturbances which are not accounted for in the related blind grasping work \cite{Kawamura2013, Tahara2010}.
\end{remark}

\begin{remark}
Assumption \ref{asm:full rank Jh} requires the hand to not be redundant to avoid any motion of the joints in the null space of $J_h$. Assumption \ref{asm:full rank Jh} in Theorem \ref{thm:Main case 1} can be relaxed to consider $m \geq 3n$  with the additional requirement that viscous friction of the form $\myvar{\tau}_e = -\beta \myvardot{q}$ for $\beta \in \mathbb{R}_{>0}$ acts on the joints. Viscous friction can be artificially introduced into the controller, but is inherent in real systems. The damping of motion in the null space due to viscous friction ensures Lemma \ref{lemma:no zero dynamics} holds, and equivalent stability results follow. Note if $m > 3n$, then the generalized inverse of $J_h$ is used in place of $J_h^{-1}$ and the proposed control \eqref{eq:proposed control}, \eqref{eq:PID control} is unchanged. This widens the applicability of the proposed method to redundant robotic hands.
\end{remark}

\subsection{Gain Tuning}\label{ssec:gain tuning}

Theorem \ref{thm:Main case 1} and Corollary \ref{cor:exponential stability} ensure the existence of PID gains to guarantee semi-global asymptotic and semi-global exponential stability, respectively, for object manipulation. In \cite{Alvarez-Ramirez2000}, a systematic tuning method was presented for PID control in which $K_p, K_i, K_d$ are parameterized by a single variable, $\varepsilon \in \mathbb{R}_{>0}$. That approach, which is adopted here, restricts the degrees of freedom in choosing the gains to facilitate the design of the controller without compromising stability. Let the $K_p, K_i, K_d$ gains be defined by:
\begin{subequations} \label{eq:PID gains as fl gains}
\begin{align}
K_p &=  \apmat{M}(K_1 + \frac{1}{\varepsilon} K_2) \\
K_i &= \frac{1}{\varepsilon} \apmat{M} K_1 \\
K_d &= \apmat{M} (K_2 + \frac{1}{\varepsilon} I_{6\times 6})
\end{align}
\end{subequations}
The structure defined in  \eqref{eq:PID gains as fl gains} facilitates the choice of each gain parameter. The gains $K_1$ and $K_2$ relate to the behavior of a linear system, and can be chosen based on the desired closed loop time constant and damping coefficient \cite{Alvarez-Ramirez2000}. The constant parameter $\apmat{M}$ can be chosen as per Remark \ref{rm:Mbar condition}. The parameter $\varepsilon$ dictates the size of the region of attraction. Once $K_1, K_2$ are defined, $\varepsilon$ is solely responsible for the system's transient response. Thus Theorem \ref{thm:Main case 1} and Corollary \ref{cor:exponential stability} can be re-stated under the restricted tuning guidelines.

\subsection{Extension to Additional Sensing Modalities}\label{ssec:additional sensing}

In much of the related work, the manipulation controllers require methods of tracking a fixed point on the object. This is typically done by either attaching sensors onto the object or using sophisticated cameras and vision systems \cite{Jara2014, Ueki2011,Fan2017}. One of the main advantages of vision is to directly measure the object pose for a given manipulation task. The use of the virtual frame, although acceptable in many applications, does not exactly track the object motion \cite{Wimbock2012}. The reason for this is due to rolling effects between the fingertip and object. Different relative curvatures between the fingertip and object will result in different object velocities for a given fingertip velocity  \cite{Murray1994}. Vision sensors bypass this issue by directly measuring object orientation error, and can be integrated into the proposed control with the same stability guarantees as follows.

In the event that object pose measurements are provided by available sensors, the proposed control can be augmented by first using the pose measurements to define the task frame. The sensors provide $\myvar{x}$ directly, where $\myvar{p}_a$ is a fixed point being tracked, while $\myvar{\gamma}_a$ is the orientation of the object about $\myvar{p}_a$ \cite{Jara2014}. From rigid body motion, it is straightforward to show  $P \myvardot{x} = \myvardot{x}_o$, which implies that $J_a = P^{-1}$. Thus the resulting control law is the same one defined by \eqref{eq:proposed control}, \eqref{eq:PID control}, albeit with a new task frame provided by the sensors. Regarding the stability analysis, the same stability results follow by setting $J_a  = P^{-1}$. 

Tactile sensors can also be incorporated into \eqref{eq:proposed control} by using $\myvar{p}_c$ in place of $\myvar{p}_t$ in $\hat{G}$ and $\hat{J}_h$. This improves the estimate of the grasp map $\apmat{G}$ and hand Jacobian $\apmat{J}_h$ to reduce the uncertainty in the system. Furthermore, as shown in the following lemma, the use of tactile sensors also relaxes Assumption \ref{asm:internal force condition}:

\begin{lemma}\label{lemma:G and Gtilde}
Consider $\hat{G}$ defined by \eqref{eq:uncertain G vision} and $G$ defined by \eqref{eq:full grasp map}, \eqref{eq:grasp map}. Suppose that $\myvar{p}_t \equiv \myvar{p}_c$. If $\bm{u}_f \in \text{Ker}(\hat{G})$, then $\bm{u}_f \in \text{Ker}(G)$.
\end{lemma} 
\begin{proof}  
The grasp map, $G$, can be re-written with respect to $\hat{G}$ by $G = \hat{G} + \delta G$ where $\delta G$ is:
\begin{equation*}\label{eq:G and tilde-G}
\delta G = \left[ \begin{array}{c c c} 0  & \ldots & 0 \\ (\redtext{\bm{p}_o - \bm{p}_a})\times  & \ldots & (\redtext{\bm{p}_o - \bm{p}_a})\times   \end{array} \right] 
\end{equation*}
The term $G \bm{u}_f$ can now be simplified using $\bm{u}_f \in \text{Ker}(\hat{G})$ as follows:
\begin{align*}
G \bm{u}_f & = \delta G \bm{u}_f \\
& = \sum_{i=1}^k \left[ \begin{array}{c} 0_{3\times 3}  \\ (\bm{p}_o - \redtext{\bm{p}_a})\times  \end{array} \right]  \bm{u}_{f_i} \\
& = \ \left[ \begin{array}{c} 0  \\ (\bm{p}_o - \redtext{\bm{p}_a}) \times \sum_{i=1}^k \bm{u}_{f_i}  \end{array} \right]  \\
& =  0
\end{align*}
The final step, $\sum_{i=1}^k   \bm{u}_{f_i}  = 0$, is true from $\hat{G} \bm{u}_f = 0$.
\end{proof}

When tactile sensors are used such that $\myvar{p}_t \equiv \myvar{p}_c$, the uncertainty in $J_h$ is only associated with uncertainties in the manipulator Jacobians $J_{si}$. Should $\apmat{J}_{si} \equiv J_{si}$, Lemma \ref{lemma:G and Gtilde} allows for more general internal force controllers that may not be constant at $\myvar{e} = 0$. This is due to the fact that since $\myvar{u}_f \in \text{Ker}(G)$, any changes in $\myvar{u}_f$ do not affect the system dynamics.

\subsection{Existing Disturbance Compensators}\label{ssec:exogenous dist comp}

In some cases knowledge of the disturbances that act on the hand-object system may be known a priori. For example, gravity compensation is typically used to compensate for gravity disturbances on the hand, however gravity compensation requires a model of the hand that will inevitably not exactly match the physical system. Similarly other disturbance models may be known, but will not completely cancel out the external disturbance. In order to incorporate these existing compensators, the exogenous input is defined as: 

\begin{equation}\label{eq:ue for dist comp}
\myvar{u}_e = -\apvar{\tau}_e - \apmat{J}_h^{T} \apmat{G}^\dagger  \apvar{w}_e
\end{equation}
where $\apvar{\tau}_e \in \mathbb{R}^m$ and $\apvar{w}_e \in \mathbb{R}^6$ are the respective approximations of $\myvar{\tau}_e$ and $\myvar{w}_e$. 

Under the condition that $\myvardot{x} = 0 \implies \apvardot{\tau}_e, \apvardot{w}_e = 0$, and based on the previous analysis, it is clear that the superposition of $\myvar{u}$ with $\myvar{u}_e$ satisfies the conditions of Theorem \ref{thm:Main case 1}. As such, $\myvar{u}_e$ can be incorporated into the proposed control to augment the control performance with existing model knowledge. The effect of $\myvar{u}_e$ is then to reduce the uncertainty that is otherwise compensated for by the integral action of \eqref{eq:PID control}. In practice, all objects have mass and thus it is advantageous to define $\apvar{w}_e$ as $\apvar{w} = \hat{m}_o \myvar{g}$, where $\myvar{g} \in \mathbb{R}^3$ is the gravity vector and $\hat{m}_o$ is an approximation of the object mass.

Insofar, the proposed control can reject disturbances that satisfy the continuously differentiable condition as per Assumption \ref{asm:constant disturbance}. In practice, robotic hands may be subject to static friction, which is a dead-zone disturbance that does not satisfy the smoothness property of Assumption \ref{asm:constant disturbance}. With the integral action used in the proposed controller, static friction can cause integrator wind-up that compromises the stability properties of the proposed control. Thus a static friction compensation method may be required for implementation of the proposed control in robotic hands.

There exist many compensation techniques in the literature to handle static friction \cite{Mohammadi2013,Armstronghelouvry1994}. Model-based techniques directly cancel out the effects of static friction, but require significant calibration effort. One non-model-based method is a dither-based static friction compensation. In the dither method, a small sinusoidal dither signal is added to the control to periodically perturb the system. The augmentation of the dither static friction compensating control to the proposed control law is:
 \begin{equation}\label{eq:manipulation control with dither}
\myvar{u} = \apmat{J}_{h}^T \Big((P^T \apmat{G})^\dagger \myvar{u}_m + \myvar{u}_f \Big) + \myvar{d}(t)
\end{equation}
where $\myvar{d}(t) \in \mathbb{R}^m$ is a vector comprised of individual dither signals $d_j(t) \in \mathbb{R}$ for $j = 1,..,m$ defined by $d_j(t) = a_j \sin(2 \pi f t) + b_j$, where $a_j, b_j, f \in \mathbb{R}$ define the amplitude, dc offset, and frequency of the dither signals.

\begin{remark}\label{rem:dither}
The use of a dither-based static friction compensator does not satisfy the conditions for Lemma \ref{lemma:psi is constant at origin} to show that the resulting disturbances are constant at the origin. However, the use of periodic averaging \cite{Khalil2002} can be used to ensure the same stability results presented previously apply to a prescribed error bound for such dither-based compensation schemes.
\end{remark}

\section{Results} \label{sec: results}

The proposed controller provides in-hand manipulation for tactile-based blind grasping and is robust to model uncertainties. This section demonstrates the application of this robust controller through numerical simulation and hardware implementation. In the numerical simulation, the proposed control is compared to an existing controller from the literature \cite{Tahara2010} in the presence of unknown external disturbances. Additionally, results from Section \ref{ssec:gain tuning} are applied to demonstrate how the systematic gain tuning and semi-global properties presented improve the response of the system despite the robotic hand being deprived of grasp information. Finally, the proposed control is implemented on hardware to demonstrate the proposed control for tactile-based blind grasping.

\subsection{Simulation}\label{ssec:sim results}

In the simulations, the robotic hand is used to grasp a rectangular prism with different object masses, and whose center of mass, $\myvar{p}_o$, is purposefully offset from the grasp centroid, $\myvar{p}_a$. The robotic hand only has access to proprioceptive measurements $\myvar{q}$, $\myvardot{q}$, and reference $\myvar{r}$. The purpose of these results are to demonstrate the robustness of the proposed control to uncertainties in $M_{ho}$, $C_{ho}$, $\myvar{w}_e$, $\myvar{\tau}_e$, $G$, and $J_h$. The proposed control is compared to the existing control from \cite{Tahara2010}. 

The robotic hand used in the simulations is the Allegro Hand \cite{Bae2012a}, which is grasping the rectangular prism as depicted in Figure \ref{fig.initgrasp}. The Allegro Hand is a fully-actuated robotic hand, and the simulated version has 3 fingers with 4 degrees of freedom each. The fingertips are modeled as hemispheres with a radii of $0.01$m. The parameters of the hand are listed in Table \ref{table:allegro hand parameters}, and the object/simulation parameters are listed in Table \ref{table:parameters}. 

\begin{table}[h!]\hspace*{-1cm}
\centering
\caption{Allegro Hand: Model Parameters} \label{table:allegro hand parameters}
\begin{tabular}{c|ccc}    \toprule
 \textbf{Dimensions} (m) & \emph{Link 1} & \emph{Link 2} & \emph{Link 3} \\\midrule
Length (index, middle, ring)    & $0.0540$ & $0.0384$ &  $0.0250$  \\
Length (thumb)   & $ 0.0554 $ & $0.0514 $ & $0.0400$  \\
Width/Depth (all) & $0.0196$ & $0.0196$ & $0.0196$\\
\midrule
\textbf{Mass} (kg) & \emph{Link 1} & \emph{Link 2} & \emph{Link 3} \\\midrule
(index, middle, ring) & $0.0444 $ & $0.0325$ & $0.0619$  \\
 (thumb) & $0.0176$ & $0.0499$ & $0.0556$  \\
\bottomrule
\end{tabular}
\end{table}

\begin{table}[h!]\hspace*{-1cm}
\centering
\caption{Simulation Parameters} \label{table:parameters}
\begin{tabular}{c c}    \toprule
Object dimensions & $ 0.20 \ \text{m} \times 0.0354 \ \text{m} \times 0.0354 \ \text{m}$ \\
Object moment of inertia & $\text{diag}([0.0418, 0.6876,$\\ & $0.6876]) \times10^{-3} \text{kg} \text{m}^2$ \\
Initial $\myvar{p}_o$ & $(-0.0258,0.0161, 0.0978) \ \text{m}$  \\
Initial $\myvar{p}_a$ & $(0.0227, 0.0245, 0.1044) \text{m}$\\
Initial $\myvar{\gamma}_a$ & $(0,0,0) \text{rad}$\\
$\myvar{\tau}_e$ & $-.001*I_{m\times m} \myvardot{q} \ \text{Nm}$\\
$\myvar{w}_e$ & $ (0, 0, -m_o*9.81, 0, 0, 0)\ \text{N}$\\\bottomrule
\end{tabular}
\end{table}

The initial grasp shown in Figure \ref{fig.initgrasp} is purposefully offset from the object center of mass, to accentuate effects of an unknown center of mass with gravity for an unknown object mass, $m_o \in \mathbb{R}$. Note this center of mass offset introduces error between the approximate grasp map $\apmat{G}$ and true grasp map $G$. Furthermore, the lack of tactile measurements introduces error in the hand kinematic approximation $\apmat{J}_h$, and additional error in $\apmat{G}$. Object masses of $0.05$, $0.10$, and $0.20$ kg are used, which result in different $M_{ho}$, $C_{ho}$, $\myvar{w}_e$ that are unknown to the controller. 

The reference, $\myvar{r}$, is decomposed into $\myvar{r} = \myvar{x}(0) + \Delta \myvar{r}$, where $\myvar{x}(0)$ is the initial state defined by the initial task frame position and orientation, and $\Delta\myvar{r} \in \mathbb{R}^6$ is the desired reference change. The reference provided for the controllers is $\Delta\myvar{r} = (0, -.01, -.02, 0, 0, 0)$, which relates to translating the object $-.01$ m in the Y-direction, $-.02$m in the Z-direction, while maintaining the same X-position and initial orientation. The simulations were performed using Matlab's ode45 integrator, with a simulation time of 10 seconds. The dynamics of the hand and object in simulation are defined via \eqref{eq:handobject system}, and the no-slip condition is enforced with \eqref{eq:grasp constraint}.

\begin{figure}[hbtp]
\centering
	\subcaptionbox{\label{fig.initgrasptopiso}}
		{\includegraphics[scale=.2]{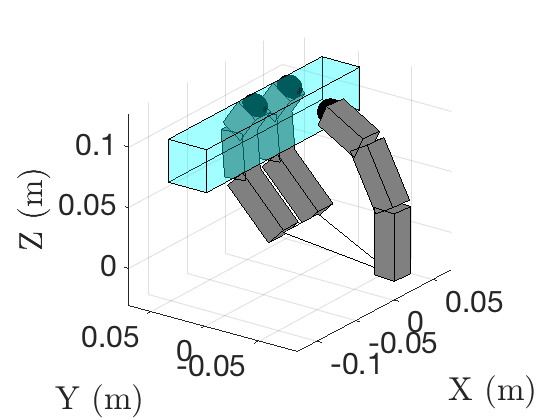}}
	\subcaptionbox{ \label{fig.initgrasptop} }
		{\includegraphics[scale=.185]{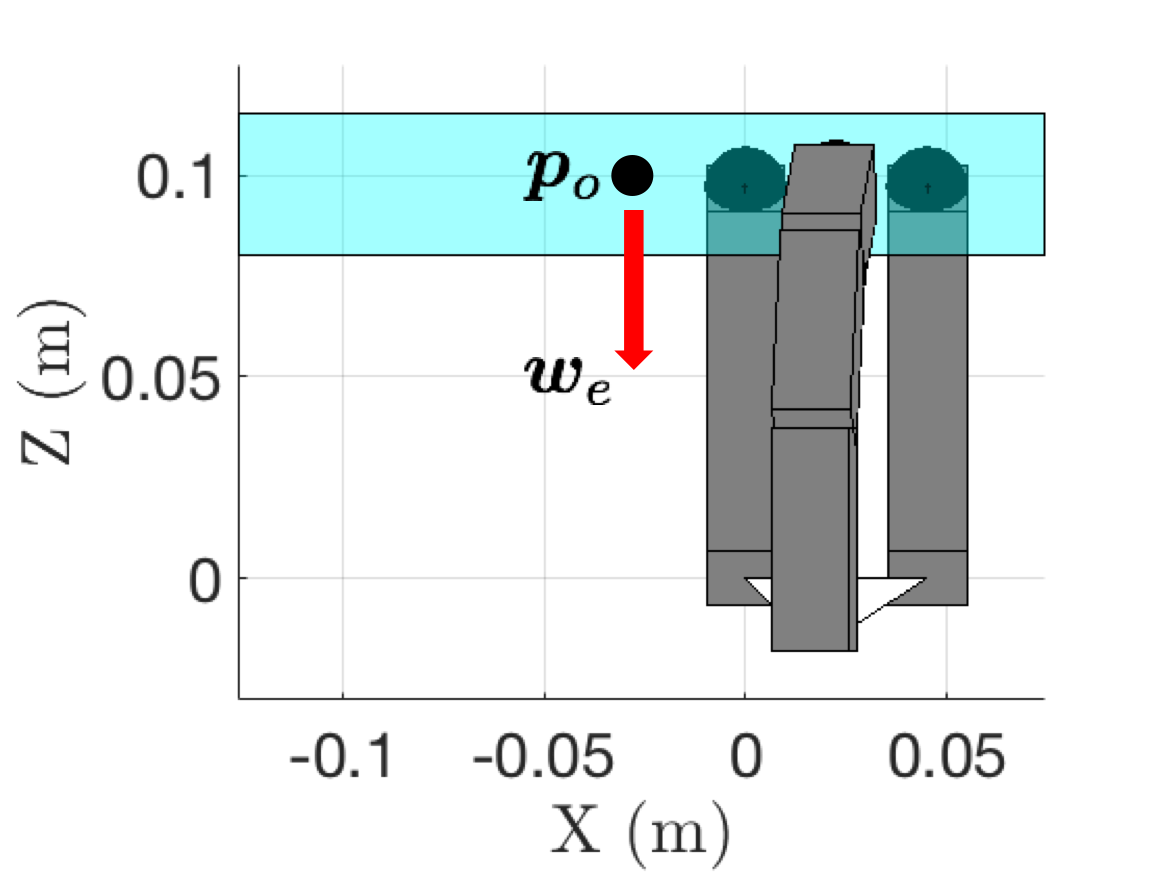}}
	\caption{Simulation setup.} \label{fig.initgrasp}
\end{figure}

The proposed control is compared to a passivity-related, conventional control defined in \cite{Tahara2010}. The proposed control \eqref{eq:proposed control}, \eqref{eq:PID control} is used with the internal force control defined by \eqref{eq:centroid uf}, which is the same internal force control used in the conventional control \cite{Tahara2010}. The gains, $K_p, K_i, K_d$  were determined by \eqref{eq:PID gains as fl gains}, with $\apmat{M} = \text{diag}([1.0, 1.0, 1.0, 0.01, 0.01, 0.01])$ , $K_1=1.0*I_{6\times 6}, K_2=2.5*I_{6\times 6}$  for $\varepsilon = 0.008$. The value of $k_f = 80$ used for both controllers was empirically chosen such that the contact points do not lose contact with the object.
   
 In each simulation, the hand-object system begins from the initial configuration depicted in Figure \ref{fig.initgrasp}. Each controller is applied at $t = 0s$ to manipulate the object to $\myvar{r}$ as the disturbances $\myvar{w}_e$, $\myvar{\tau}_e$ act on the system. The following plots show the response of the conventional and proposed controllers for different object masses. Note that in the plots, $\myvar{e}_j$ refers to the $j$th element of $\myvar{e}$.
   
   \begin{figure}[hbtp]
\centering
	\subcaptionbox{Position error. \label{fig.positionerror_bg_mo0.05} }
		{\includegraphics[scale=.2]{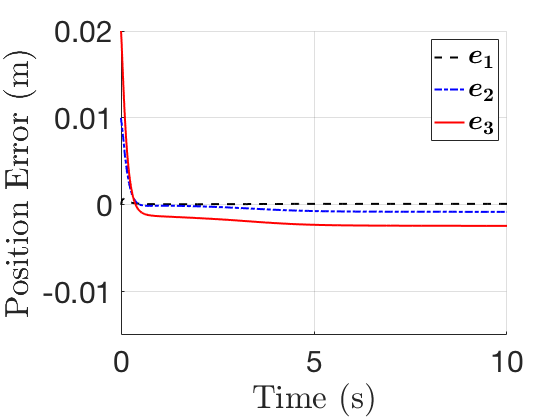}}
	\subcaptionbox{Orientation error.  \label{fig.orientationerror_bg_mo0.05} }
		{\includegraphics[scale=.2]{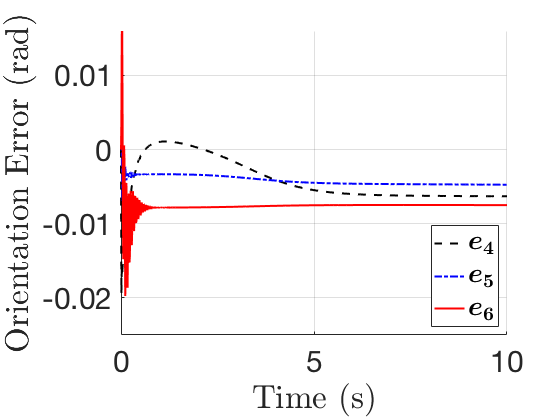}}
	\caption{Conventional control for $m_o = 0.05$ kg.} \label{fig.simulationerror_bg_mo0.05}
\end{figure}

\begin{figure}[hbtp]
\centering
	\subcaptionbox{Position error. \label{fig.positionerror_bg_mo0.1} }
		{\includegraphics[scale=.2]{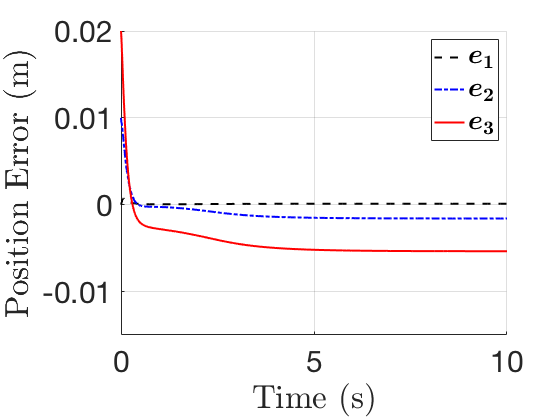}}
	\subcaptionbox{Orientation error.  \label{fig.orientationerror_bg_mo0.1} }
		{\includegraphics[scale=.2]{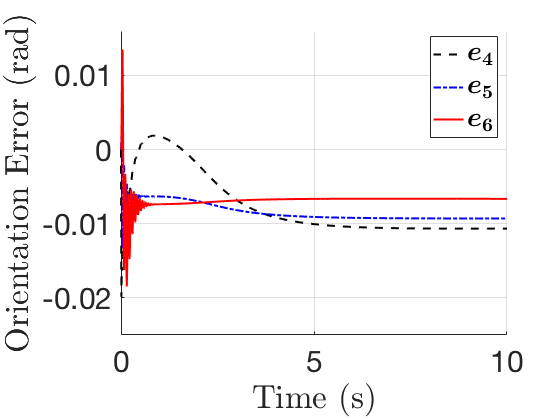}}
	\caption{Conventional control for $m_o = 0.10$ kg.} \label{fig.simulationerror_bg_mo0.1}
\end{figure}

\begin{figure}[hbtp]
\centering
	\subcaptionbox{Position error. \label{fig.positionerror_bg_mo0.2} }
		{\includegraphics[scale=.2]{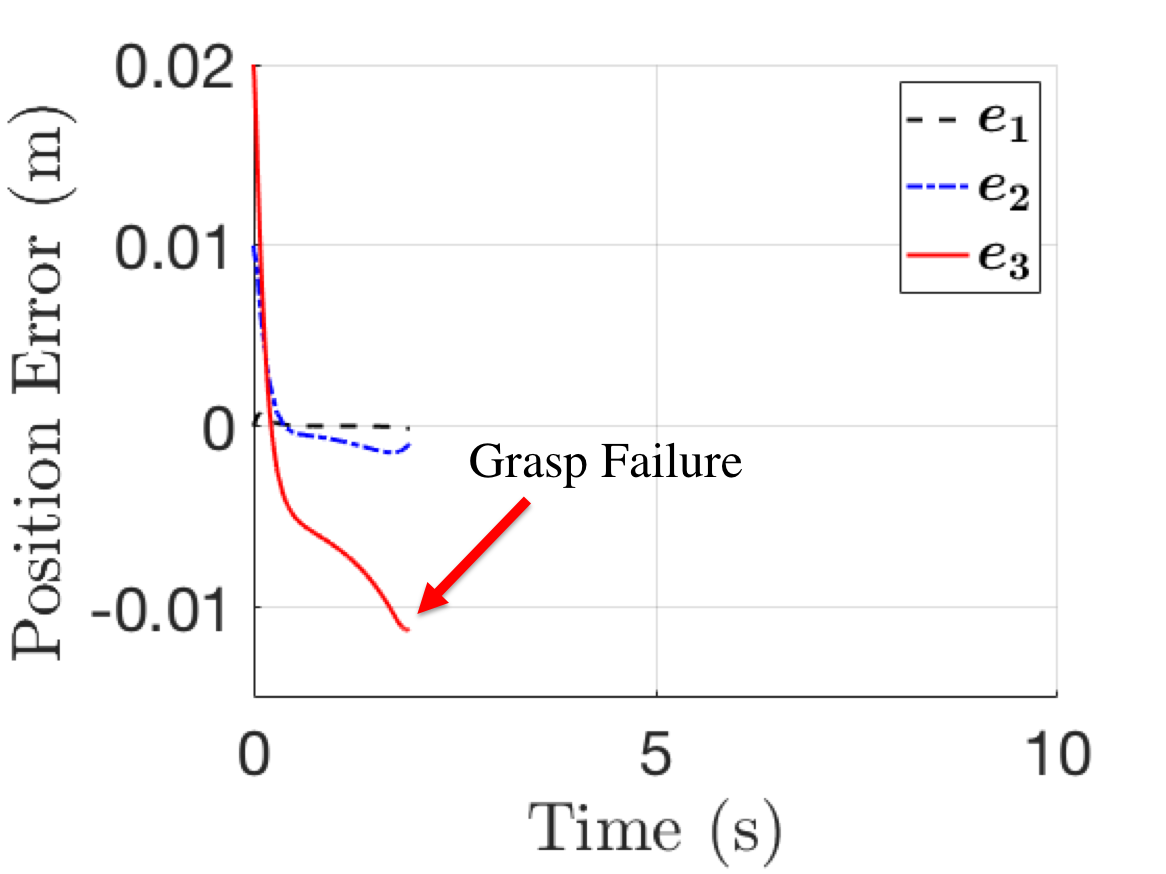}}
	\subcaptionbox{Orientation error.  \label{fig.orientationerror_bg_mo0.2} }
		{\includegraphics[scale=.2]{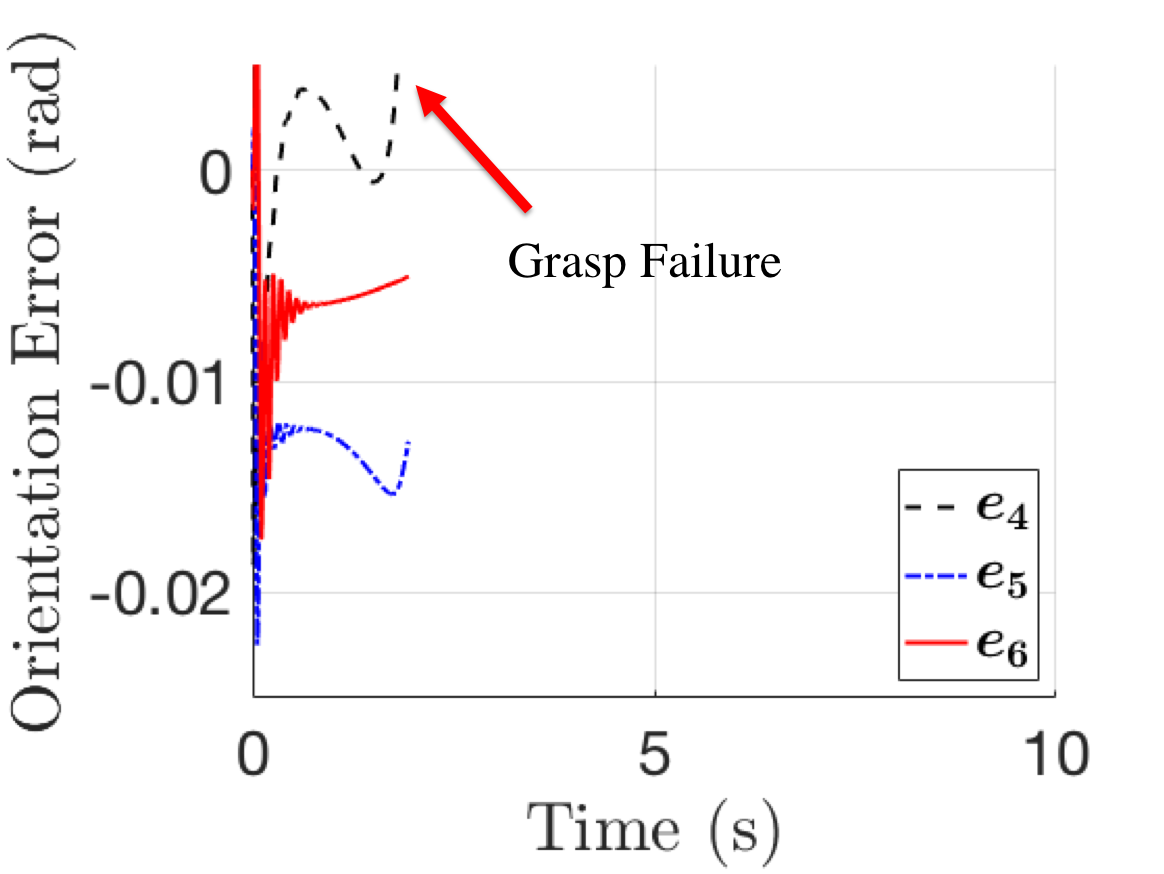}}
	\caption{Conventional control for $m_o = 0.20$ kg.} \label{fig.simulationerror_bg_mo0.2}
\end{figure}

\begin{figure}[hbtp]
\centering
	\subcaptionbox{Position error. \label{fig.positionerror_pc_mo0.05} }
		{\includegraphics[scale=.2]{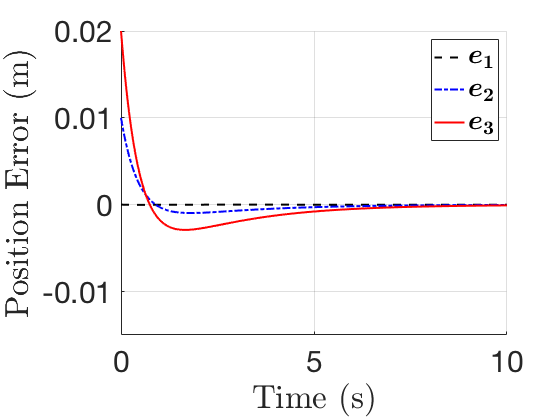}}
	\subcaptionbox{Orientation error.  \label{fig.orientationerror_pc_mo0.05} }
		{\includegraphics[scale=.2]{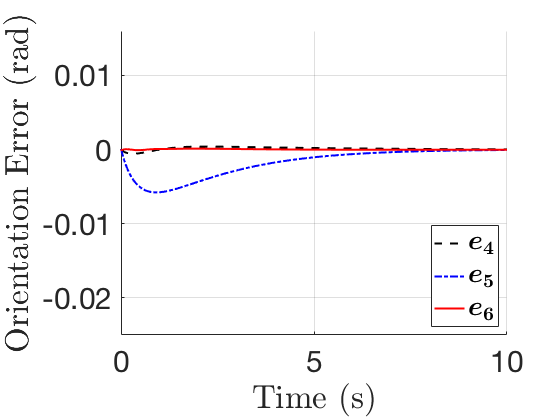}}
	\caption{Proposed control for $m_o = 0.05$ kg.} \label{fig.simulationerror_pc_mo0.05}
\end{figure}

\begin{figure}[hbtp]
\centering
	\subcaptionbox{Position error. \label{fig.positionerror_pc_mo0.1} }
		{\includegraphics[scale=.2]{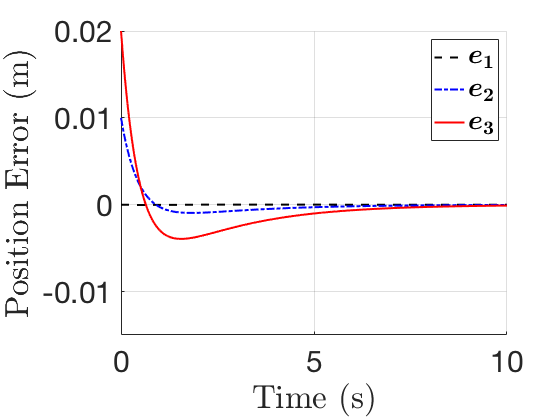}}
	\subcaptionbox{Orientation error.  \label{fig.orientationerror_pc_mo0.1} }
		{\includegraphics[scale=.2]{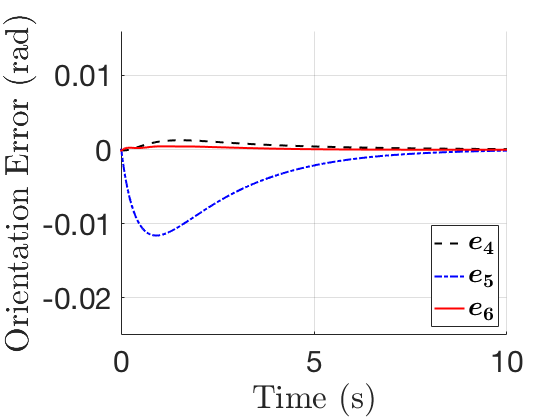}}
	\caption{Proposed control for $m_o = 0.10$ kg.} \label{fig.simulationerror_pc_mo0.1}
\end{figure}

\begin{figure}[hbtp]
\centering
	\subcaptionbox{Position error. \label{fig.positionerror_pc_mo0.2} }
		{\includegraphics[scale=.2]{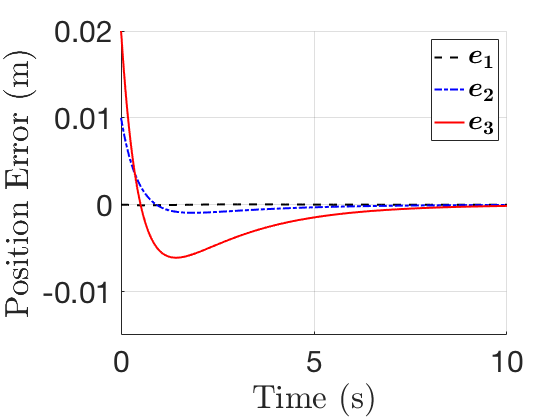}}
	\subcaptionbox{Orientation error.  \label{fig.orientationerror_pc_mo0.2} }
		{\includegraphics[scale=.2]{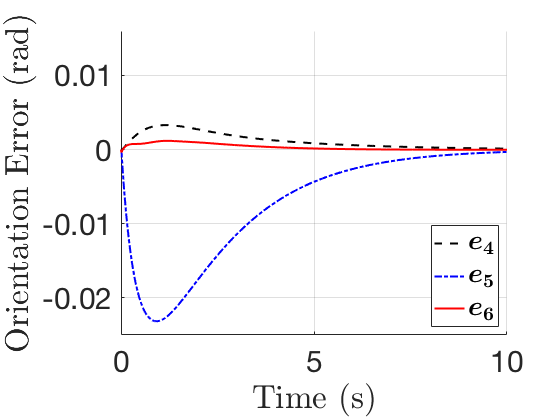}}
	\caption{Proposed control for $m_o = 0.20$ kg.} \label{fig.simulationerror_pc_mo0.2}
\end{figure}

Figures \ref{fig.simulationerror_bg_mo0.05}-\ref{fig.simulationerror_bg_mo0.2} show the responses of the conventional control law from \cite{Tahara2010}. For object masses of $0.05$ and $0.10$ kg, the controller reaches a steady-state offset of increasing magnitude for increasing object mass. This is expected as that control law neglects the effects of external disturbances, which compromises its ability to stabilize the hand-object system. For an object mass of $0.2$ kg, this issue is exacerbated as the external disturbance causes the contact points to leave the fingertip surface, causing grasp failure (i.e. the object is dropped). This is indicated by the abrupt stop in simulation at $t = 1.94$s. These results show that negligence of hand-object dynamics in tactile-based blind grasping not only results in steady-state offsets from the origin, but can ultimately result in grasp failure.

Figures \ref{fig.simulationerror_pc_mo0.05}-\ref{fig.simulationerror_pc_mo0.2} show the response of the proposed controller for the set of object masses. As expected, the proposed control law asymptotically converges to the origin despite different/unknown $M_{ho}$, $C_{ho}$, $G$, $\myvar{w}_e$, $\myvar{\tau}_e$, and uncertain $J_h$. The plots show that smaller object masses depict better transient performance due to smaller magnitude disturbance, which is aligned with intuition. The plots highlight the benefit of a robust controller that can handle unknown object masses and model uncertainty in that the gains do not need to be re-tuned despite different objects being grasped. Furthermore, with the same choice of gains, the proposed controller prevents grasp failure that occurred from the conventional control law.

Figure \ref{fig.effect of epsilon} shows how the choice of $\varepsilon$ (i.e. $K_p, K_i, K_d)$ affects the transient response of the system for an object mass of $m_o = 0.20$ kg. The figure displays the response with respect to the orientation error about the Y-axis (denoted by $\myvar{e}_5$), which is most affected by the torque disturbance of the object weight acting about the grasp centroid. However the same behavior is seen in all components of $\myvar{e}$ which are omitted here for clarity.

\begin{figure}[hbtp]
\centering
\includegraphics[scale=0.3]{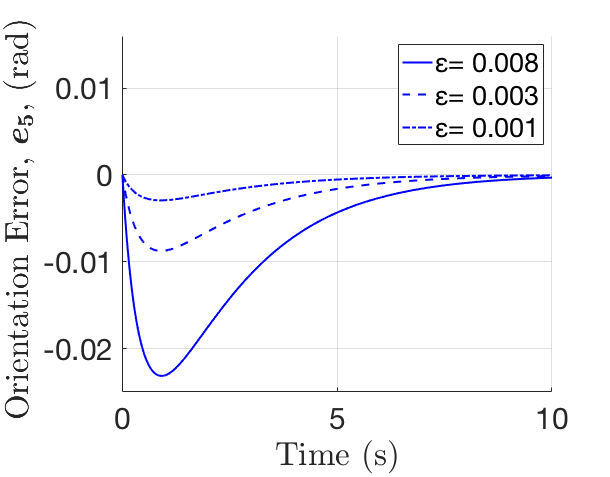}
\caption{Effect of $\varepsilon$ on transient behavior for proposed control. }  \label{fig.effect of epsilon}
\end{figure}

Figure \ref{fig.effect of epsilon} shows that for decreasing values of $\varepsilon$ (increasing $K_p, K_i, K_d$), the system has improved performance with respect to smaller overshoot and settling times. This trend is consistent with the notion from Corollary \ref{cor:exponential stability} where as $\varepsilon$ is decreased $(K_p, K_i, K_d$ increased), improved stability conditions are guaranteed in the form of asymptotic then exponential stability. This is an advantage over existing controllers \cite{Tahara2010, Wimbock2012, Kawamura2013}, in that the proposed control can achieve arbitrary control performance in the presence of external disturbances for tactile-based blind grasping. However the control performance will ultimately be limited by hardware restrictions.

 \subsection{Hardware}\label{ss:exp results}

The purpose of the hardware results is to highlight the robustness of the proposed control. This is achieved by conducting four demonstrations of the proposed control applied to objects of different mass and shape, which are unknown to the controller. In each demonstration, one parameter: object mass, object shape, or manipulation command is changed. In the first demonstration, the proposed control rotates a spherical object of mass $0.20$ kg. In the second demonstration, the same control rotates a \textit{cube} object of mass $0.20$ kg. In the third demonstration, the same control rotates a spherical object of \textit{mass $0.09$ kg}. In the final demonstration, the same control \textit{translates} a spherical object of mass $0.20$ kg. Note the same control and gain values are fixed for \textit{all} demonstrations to highlight the robustness of the proposed method to different object masses/shapes.
 
The proposed controller was applied to the Allegro Hand hardware to demonstrate the robustness properties of the proposed controller for objects of different mass and shape. The Allegro Hand setup, shown in Figure \ref{fig.AllegroHandsetup}, includes a NI USB-8473s High-Speed CAN, which operates at a fixed sampling frequency of 333 Hz. The maximum torque that can be applied by each motor is $0.65$ Nm. The Allegro Hand is fully-actuated and consists of 4 fingers each with 4 degrees of freedom. 

The inertial frame, $\mathcal{P}$, is fixed to the palm of the hand as shown in Figure \ref{fig.Allegrohandsetup1}. Let the reference with respect to $\mathcal{P}$ be $\myvar{r} = \myvar{x}(0) + \Delta \myvar{r}$, where $\myvar{x}(0)$ is the initial configuration of the task frame defined by \eqref{eq:define pa for blind grasping}, \eqref{eq:define Rpa for blind grasping}, and $\Delta \myvar{r}$ is the commanded reference change.  The change in reference is parameterized by $\Delta \myvar{r} =$ $(0,0,r_z, 0,0,r_\psi)$, where $r_z \in \mathbb{R}$ denotes a translation in the Z-axis and $r_\psi \in \mathbb{R}$ denotes a rotation about the Z-axis.

The objects used to conduct the demonstration are a 3D-printed sphere of radius $0.0375$ m and a 3D-printed cube of length $0.060$ m (see Figure \ref{fig:objects}). The sphere has a hollow interior to adjust its mass between $0.09$ kg and $0.20$ kg. Both objects have high friction coefficients to prevent slip during the manipulation motion. For each case study, the objects are first placed in the robotic hand grasp prior to implementing the control. Once the grasp is formed, the proposed controller is implemented and commanded to the desired reference. The rotation reference setpoint is $r_\psi = 0.6$ $\pm 0.06$ rad ($r_z = 0.0$ $\pm 0.008$ m). The translation reference setpoint is $r_z = 0.04$ $\pm 0.004$ m ($r_\psi = 0.0$ $\pm 0.12$ rad).

\begin{figure}[hbtp]
\centering
	\subcaptionbox{Top View \label{fig.Allegrohandsetup1} }
		{\includegraphics[scale=.23]{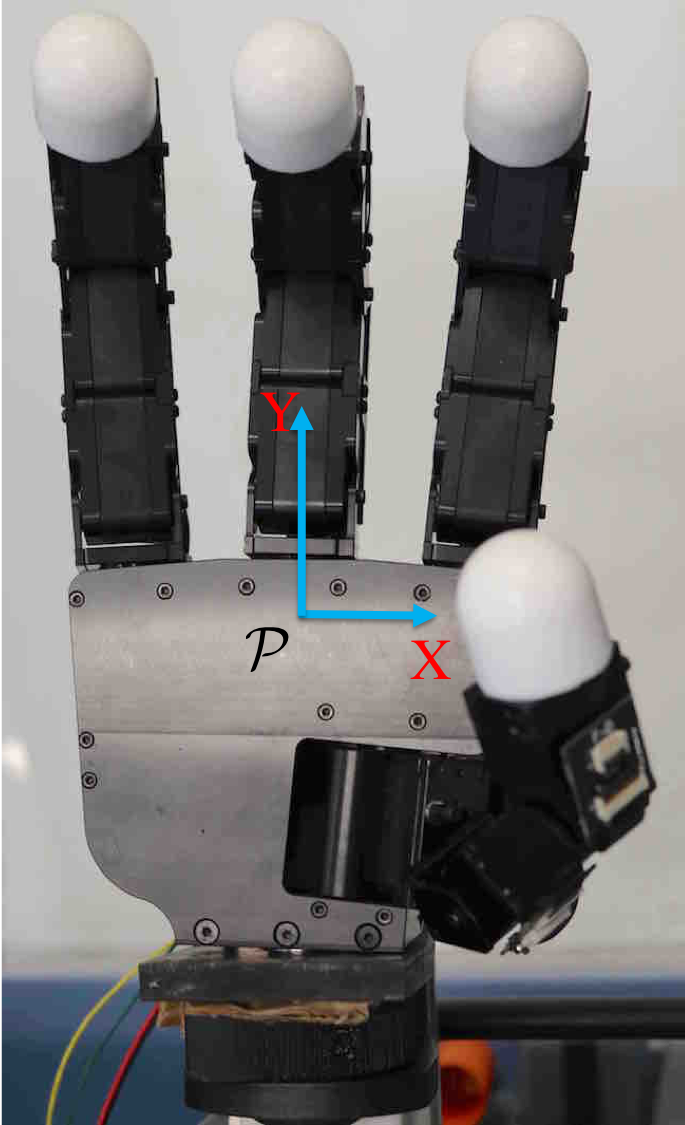}}
	\subcaptionbox{ Side View \label{fig.Allegrohandsetup2} }
		{\includegraphics[scale=.237]{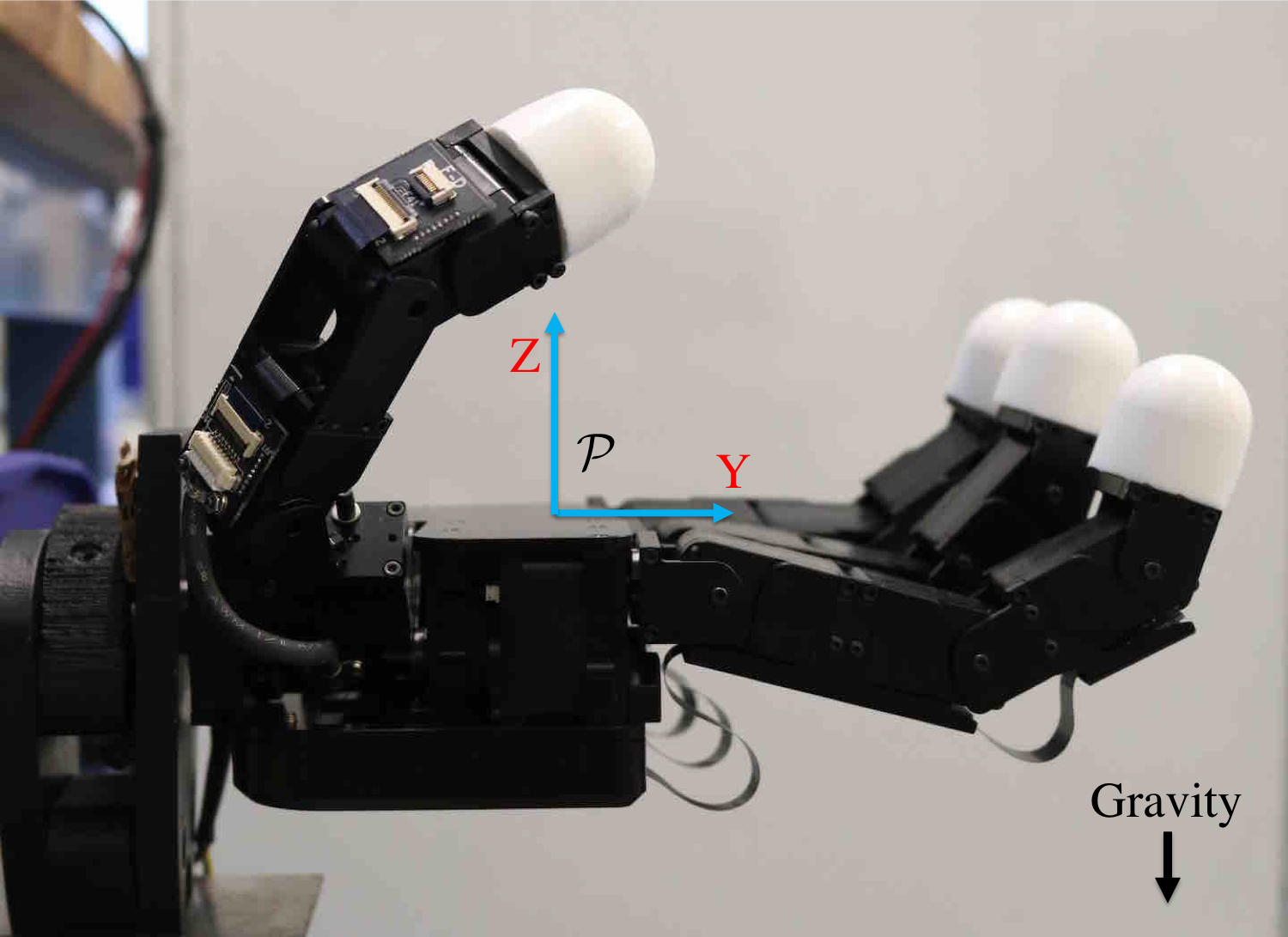}}
	\caption{Allegro Hand setup.}  \label{fig.AllegroHandsetup}
\end{figure}

\begin{figure}[hbtp]
\centering
	\subcaptionbox{Sphere and Cube \label{fig:both objects} }
		{\includegraphics[scale=.093]{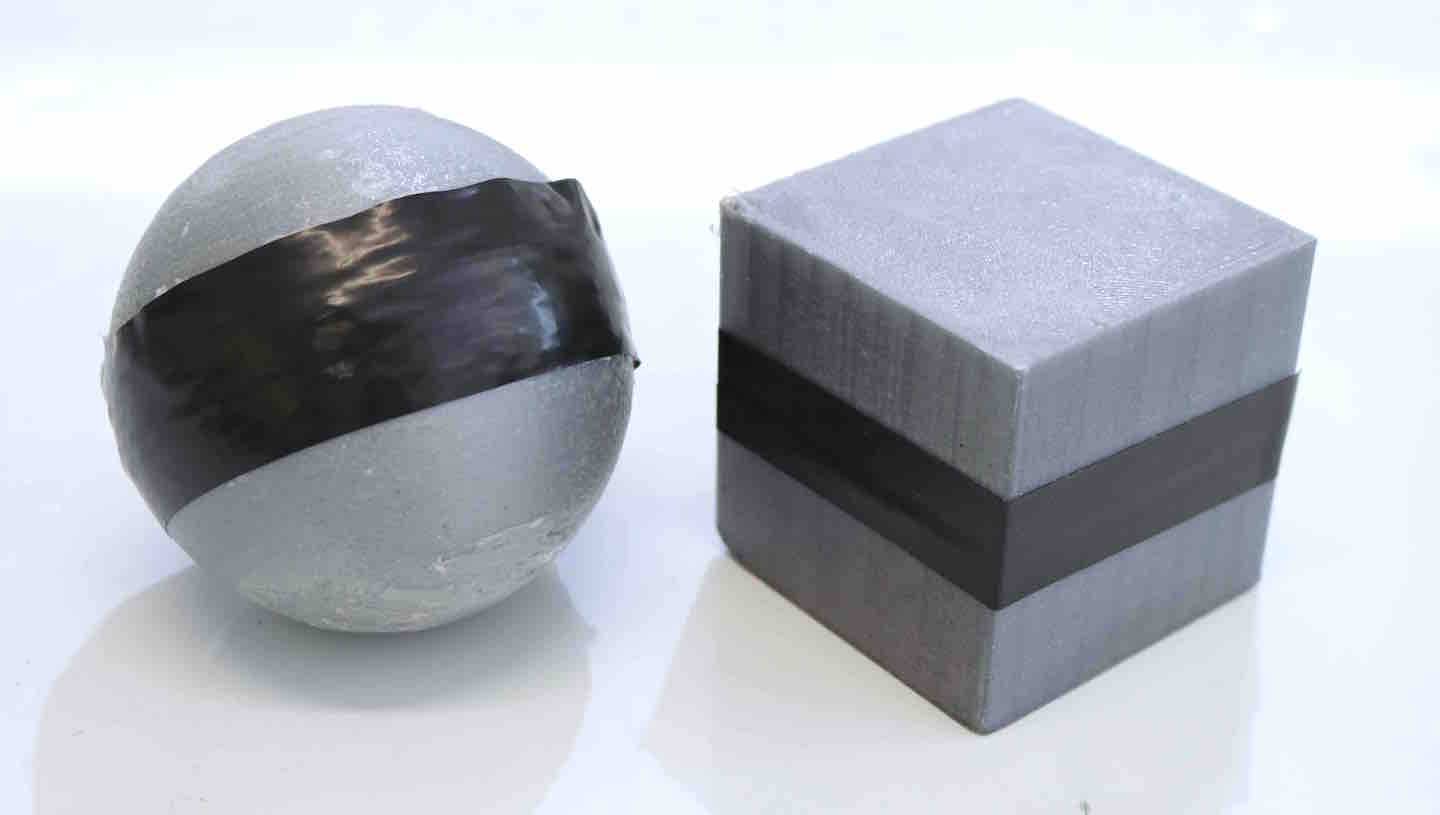}}
	\subcaptionbox{ Hollow sphere with weights \label{fig:open ball} }
		{\includegraphics[scale=.093]{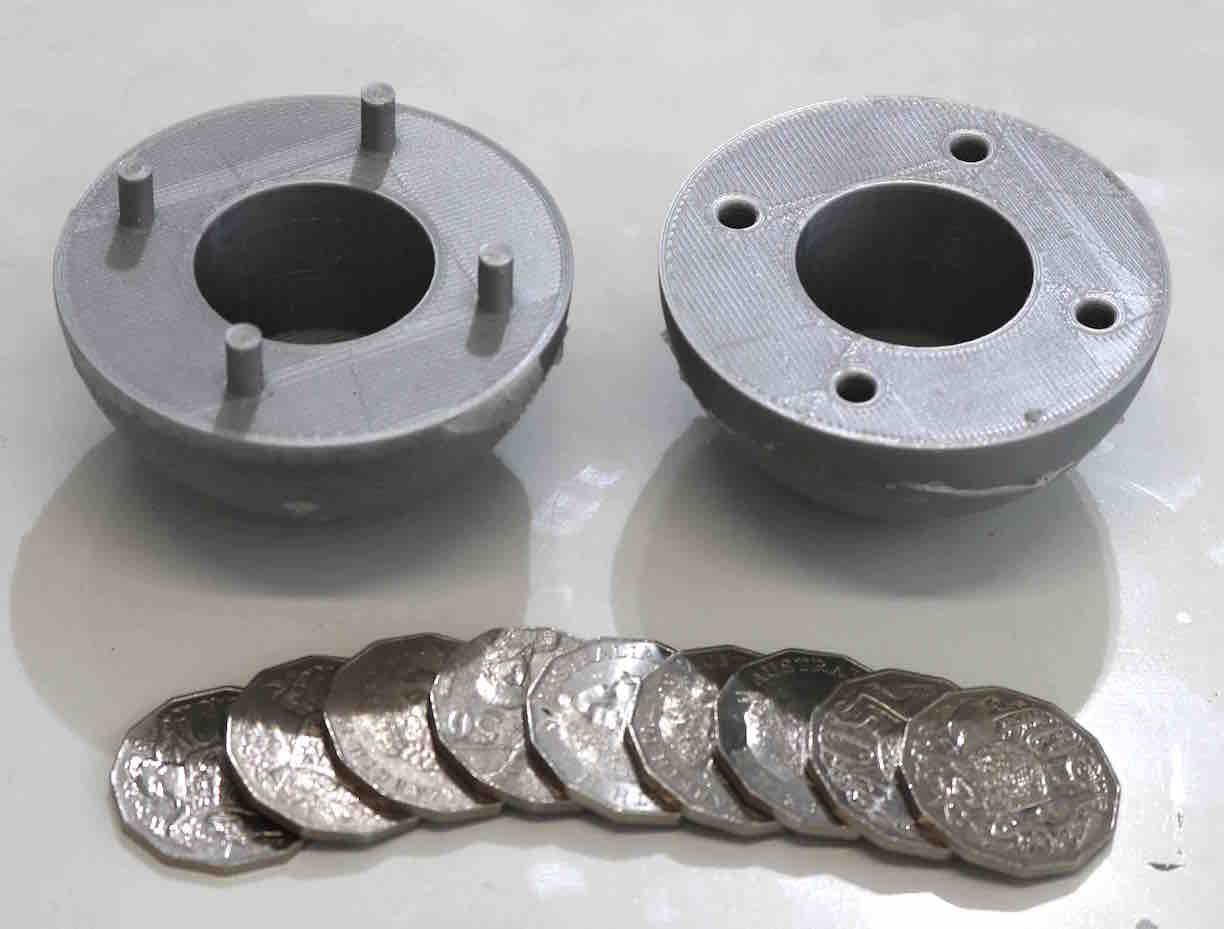}}
	\caption{Objects.}  \label{fig:objects}
\end{figure}

The controller used in the demonstration consists of the proposed control \eqref{eq:proposed control}, \eqref{eq:PID control}, \eqref{eq:centroid uf} with the exogenous disturbance compensators discussed in Section \ref{ssec:exogenous dist comp}. First, a gravity compensation component was augmented to the proposed control to account for the effect of gravity on the hand. Second, a dither-based static friction compensator was also augmented to the control due to the large presence of static friction found in the Allegro Hand. Finally, a nominal object weight compensator is augmented to the control to account for real-world scenarios in which all objects have mass. It is important to emphasize that the nominal object mass of $\hat{m}_o = 0.10$ kg is purposefully offset from the true object masses of $0.20$ kg (for demonstrations 1, 2, and 4) and $0.09$ kg (for the third demonstration) to demonstrate how the proposed control compensates for object mass uncertainty. The implemented controller is:
 \begin{equation}\label{eq:experiment blind grasp control}
\myvar{u} = \apmat{J}_{h}^T \Big((P^T \apmat{G})^\dagger \myvar{u}_m + \myvar{u}_f \Big) + \apvar{\tau}_g(\myvar{q}) + \myvar{d}(t) + \apmat{J}_h^T (P^T \apmat{G})^\dagger \hat{m}_o \myvar{g}
\end{equation}
where $\apvar{\tau}_g(\myvar{q}) \in \mathbb{R}^m$ is the approximate torque induced by gravity acting on the hand, and $\myvar{d}(t)$ is the dither signal added to compensate for static friction, and $\myvar{g} \in \mathbb{R}^3$ is the inertial gravity vector. The parameters associated with the dither signal, $\myvar{d}$, are listed in Table \ref{table:dither parameters} of the Appendix. Note the dither signal and integrator of the control \eqref{eq:experiment blind grasp control} is only applied outside of the prescribed tolerance for the setpoint reference.

The PID gains of the proposed control are:
\begin{eqnarray*}
K_p & = & \text{diag}[(500, 500, 500, 0.28, 0.28, 0.28)] \\
K_i & = & \text{diag}[(50, 50, 50, 0.6, 0.6, 0.6)] \\
K_d & = & \text{diag}[(0.008, 0.008, 0.008, 0.16, 0.16, 0.16)]
\end{eqnarray*}

\begin{figure}[ht]
\centering
	\subcaptionbox{Initial configuration \label{fig.initial sphere} }
		{\includegraphics[scale=.0951]{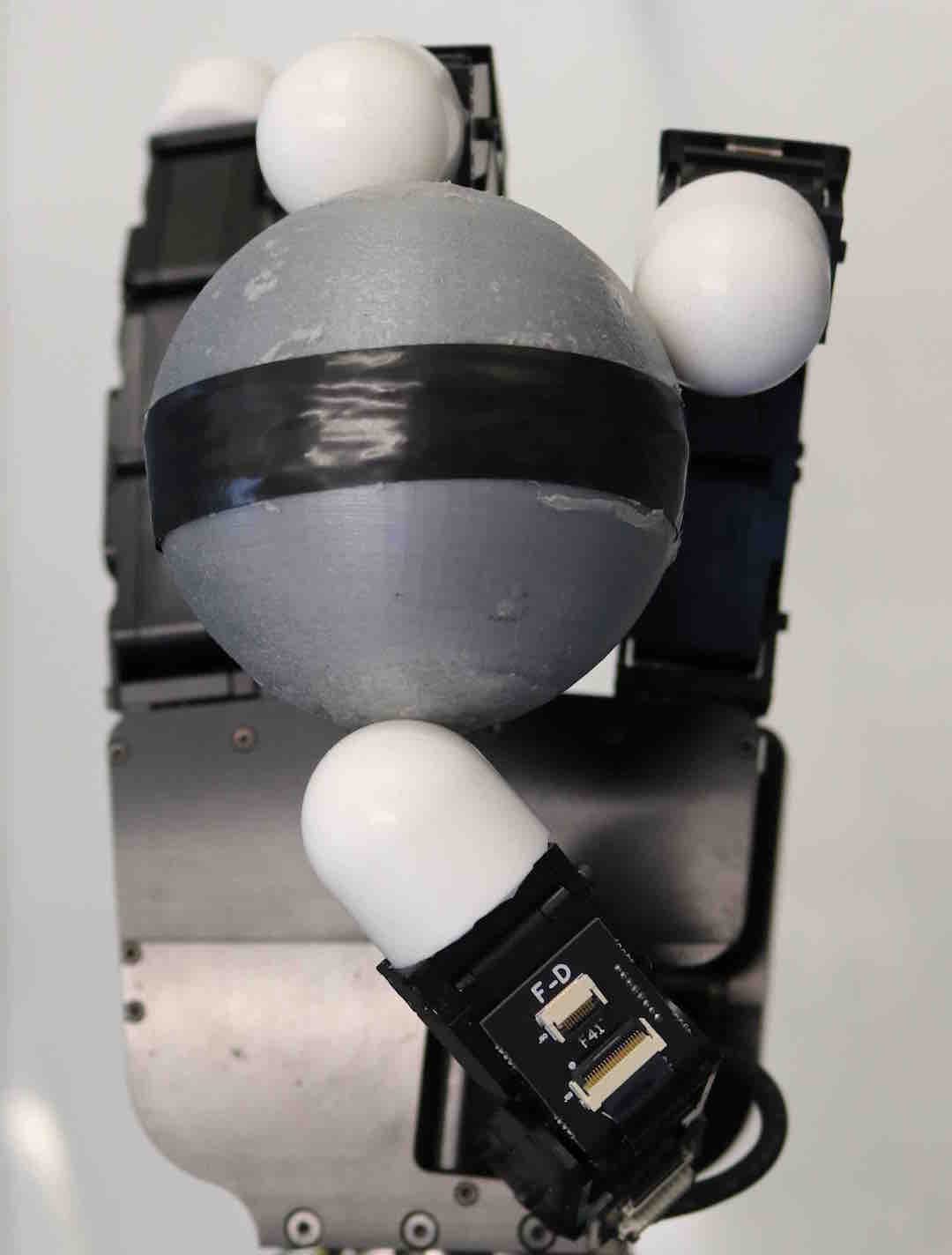}}
	\subcaptionbox{ Final configuration \label{fig.final sphere} }
		{\includegraphics[scale=.09]{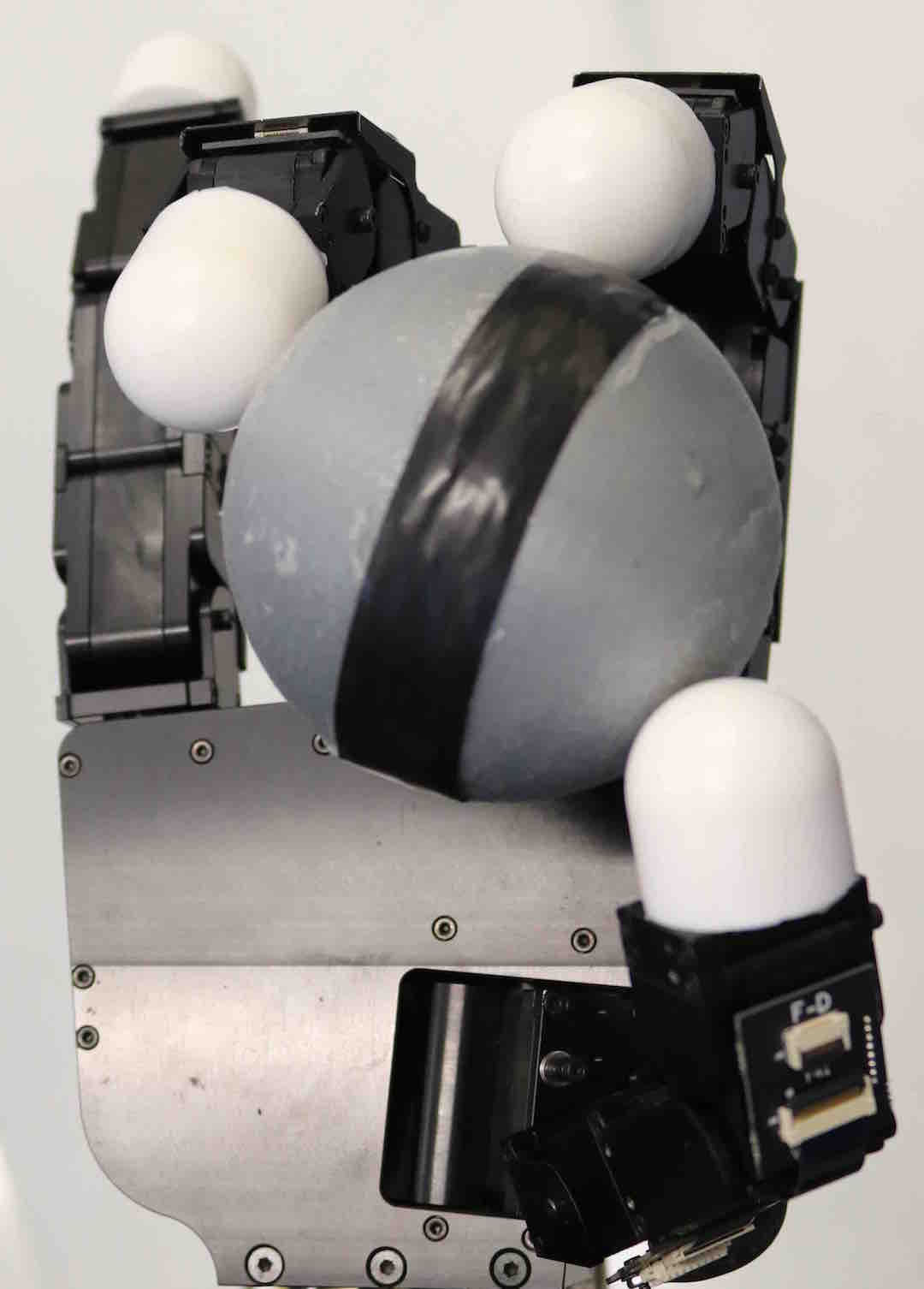}}
	\subcaptionbox{Position error.  \label{fig:cs1_position} }
		{\includegraphics[scale=.222]{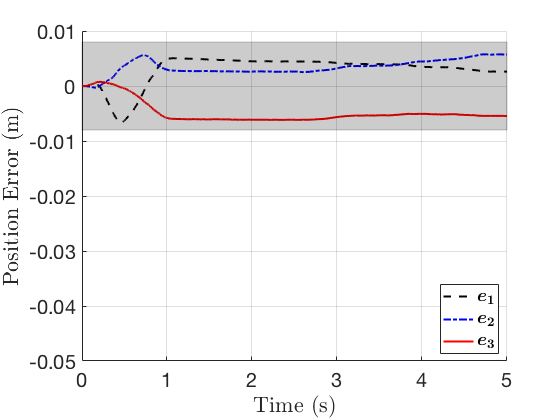}}
	\subcaptionbox{Orientation error.  \label{fig:cs1_orientation} }
		{\includegraphics[scale=.222]{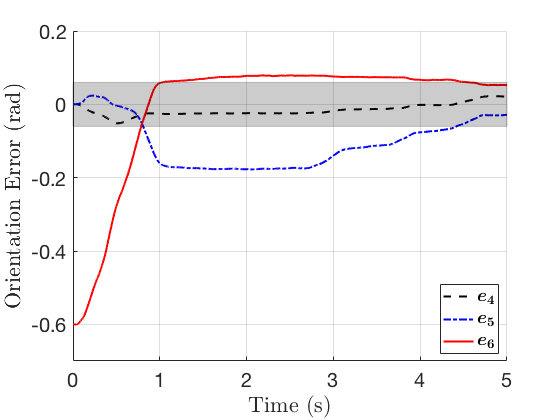}}
	\subcaptionbox{Control torque including dither signal (dither is shut off at 4.56 s when tolerance is reached).  \label{fig:cs1_torque} }
		{\includegraphics[scale=.222]{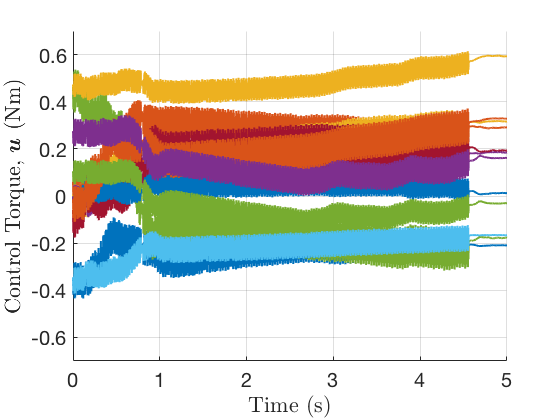}}
	\caption{Demonstration 1: Rotation of $0.20$ kg sphere to $0.6$ rad.}  \label{fig:cs1}
\end{figure}

Figure \ref{fig:cs1} shows the response of the system for the first demonstration, where the spherical object of mass $0.20$ kg is rotated. Figure \ref{fig:cs1_orientation} shows the orientation error, which depicts the proposed controller rotating the object to $r_\psi$ within the prescribed tolerance (depicted in gray). Figure \ref{fig:cs1_position} shows the position error during this manipulation motion. The plot shows that the controller also maintains object initial position, within tolerance, during the rotation maneuver. Figure \ref{fig:cs1_torque} shows the applied control torque to achieve the manipulation command. The plot shows the superposition of the proposed control with the dither signal applied by all motors. Note that during the manipulation, static friction causes minor halting behavior which is seen between $t = 1$s and $t = 2.9$s for $e_5$. The use of the dither signal mitigates this behavior by consistently vibrating each joint. The results show that the proposed control is able to compensate for the discrepancy between the nominal and actual object mass, and perform the desired manipulation command.

\begin{figure}[ht]
\centering
	\subcaptionbox{Initial configuration \label{fig.initial cube} }
		{\includegraphics[scale=.19]{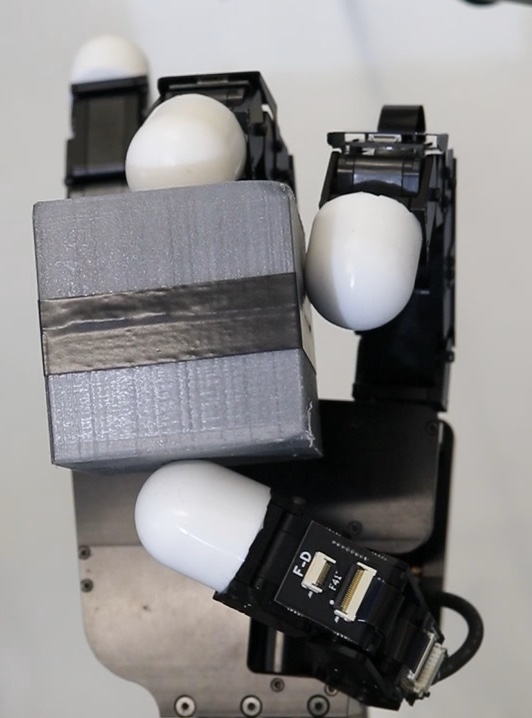}}
	\subcaptionbox{ Final configuration \label{fig.final cube} }
		{\includegraphics[scale=.19]{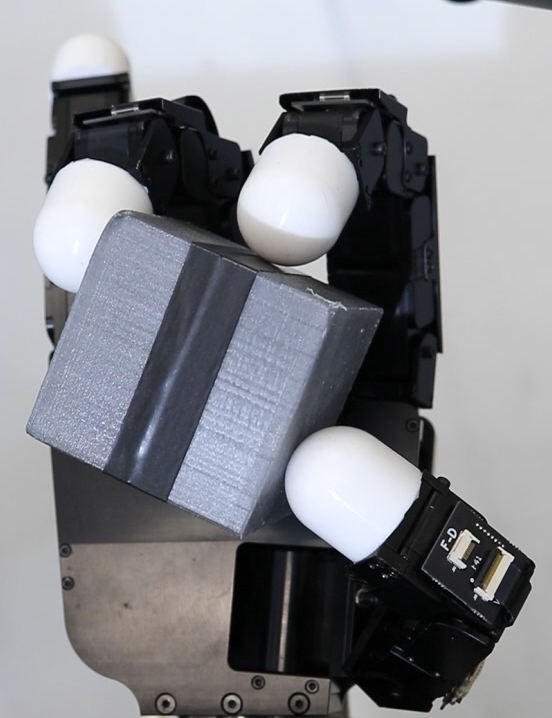}}
	\subcaptionbox{Position error.  \label{fig:cs2_position} }
		{\includegraphics[scale=.222]{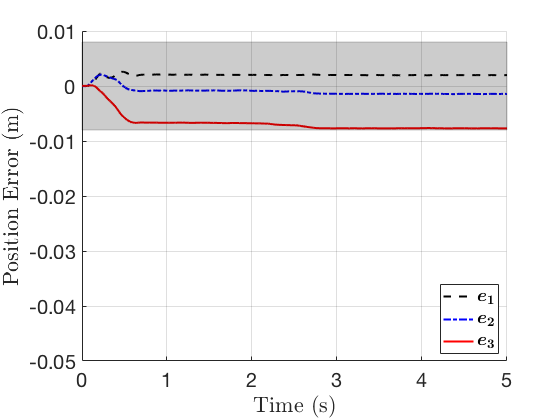}}
	\subcaptionbox{Orientation error.  \label{fig:cs2_orientation} }
		{\includegraphics[scale=.222]{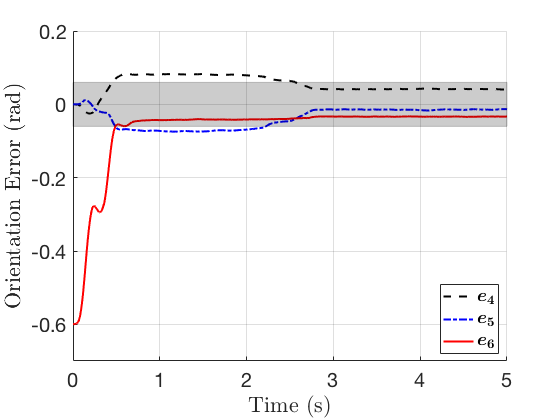}}
	\subcaptionbox{Control torque including dither signal (dither is shut off at 2.57 s when tolerance is reached).  \label{fig:cs2_torque} }
		{\includegraphics[scale=.222]{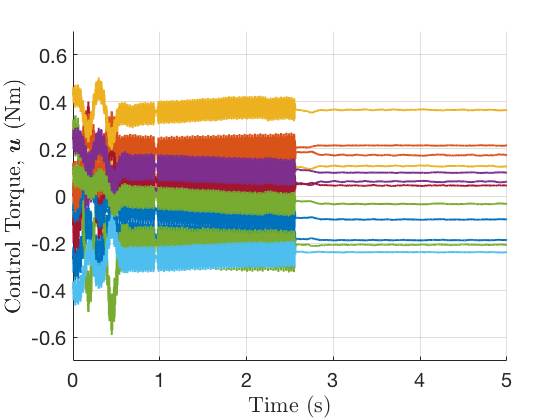}}
	\caption{Demonstration 2: Rotation of $0.20$ kg cube to $0.6$ rad.}  \label{fig:cs2}
\end{figure}

Figure \ref{fig:cs2} shows the response of the system for the second demonstration. In the second demonstration, the cube object is rotated with the same mass of $0.20$ kg used in the first demonstration. Figure \ref{fig:cs2_orientation} shows that the controller is able to rotate the cube to the desired rotation command. Figure \ref{fig:cs2_position} shows that whilst rotating the object, the controller maintains the same initial position, within tolerance. This shows that the same control from the first demonstration is able to rotate an object of different shape to the desired reference.

\begin{figure}[ht]
\centering
	\subcaptionbox{Initial configuration \label{fig.demo3 initial} }
		{\includegraphics[scale=.17]{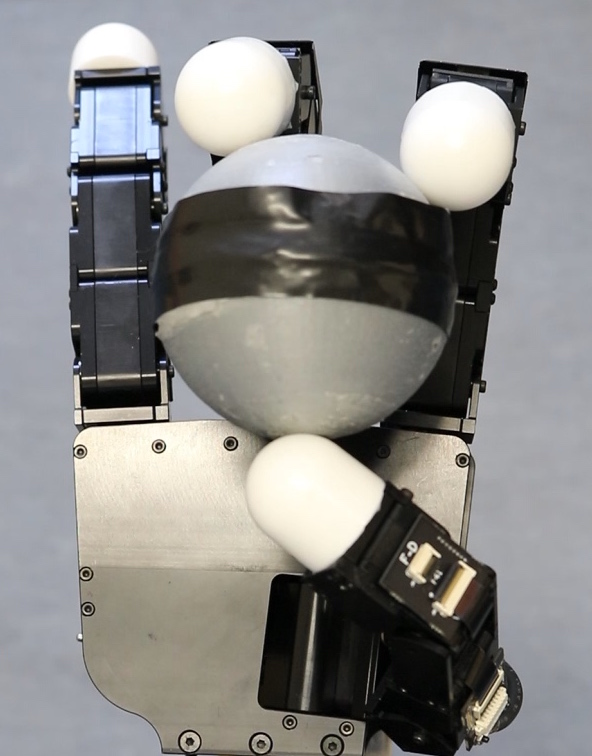}}
	\subcaptionbox{ Final configuration \label{fig.demo3 final} }
		{\includegraphics[scale=.173]{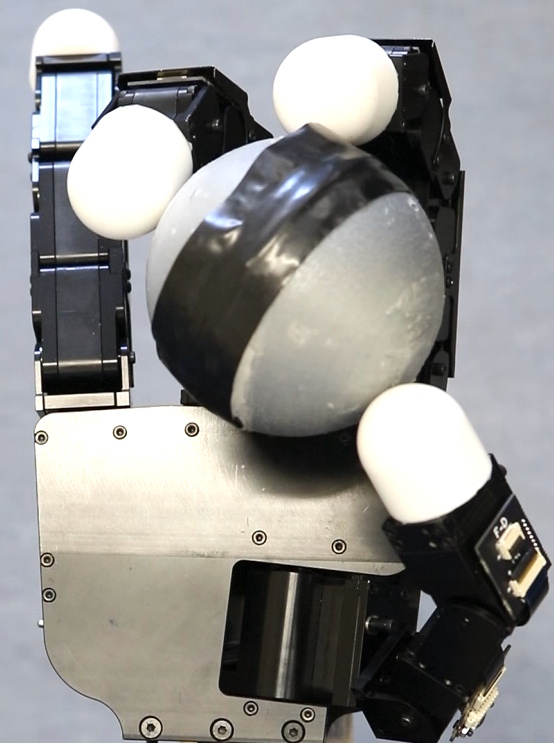}}
	\subcaptionbox{Position error. \label{fig:cs3_position} }
		{\includegraphics[scale=.222]{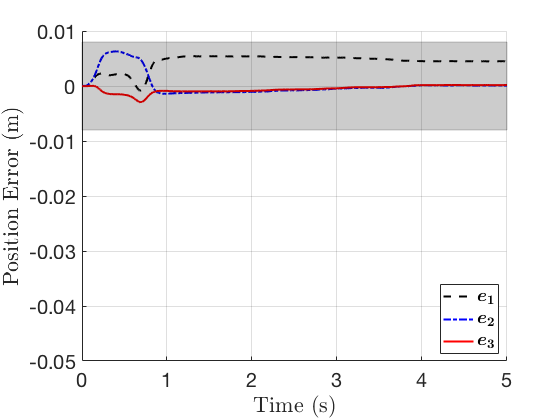}}
	\subcaptionbox{Orientation error.  \label{fig:cs3_orientation} }
		{\includegraphics[scale=.222]{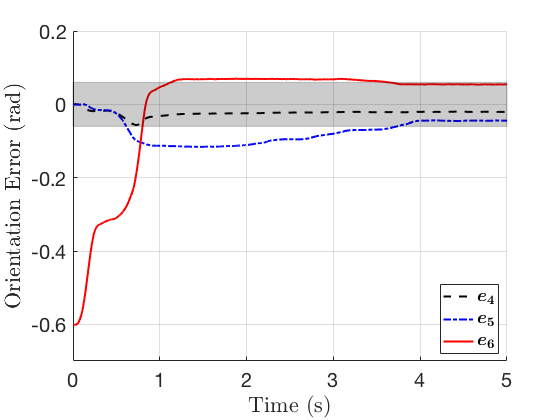}}
	\subcaptionbox{Control torque including dither signal (dither shut off at 3.74 s when tolerance is reached).  \label{fig:cs3_torque} }
		{\includegraphics[scale=.222]{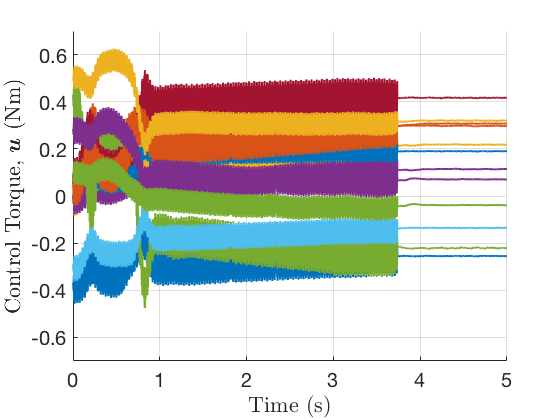}}
	\caption{Demonstration 3: Rotation of $0.09$ kg sphere to $0.6$ rad.}  \label{fig:cs3}
\end{figure}

Figure \ref{fig:cs3} shows the response of the system for the third demonstration, in which a lighter spherical object ($0.09$ kg) is rotated to the same reference used in the previous demonstrations. Figure \ref{fig:cs3_orientation} shows the orientation error converging within the desired reference tolerances. Similarly, Figure \ref{fig:cs3_position} shows the position error converging within the desired reference tolerances. This shows the ability of the control to manipulate objects of different masses.

\begin{figure}[ht]
\centering
	\subcaptionbox{Initial configuration \label{fig.demo 4 initial sphere} }
		{\includegraphics[scale=.116]{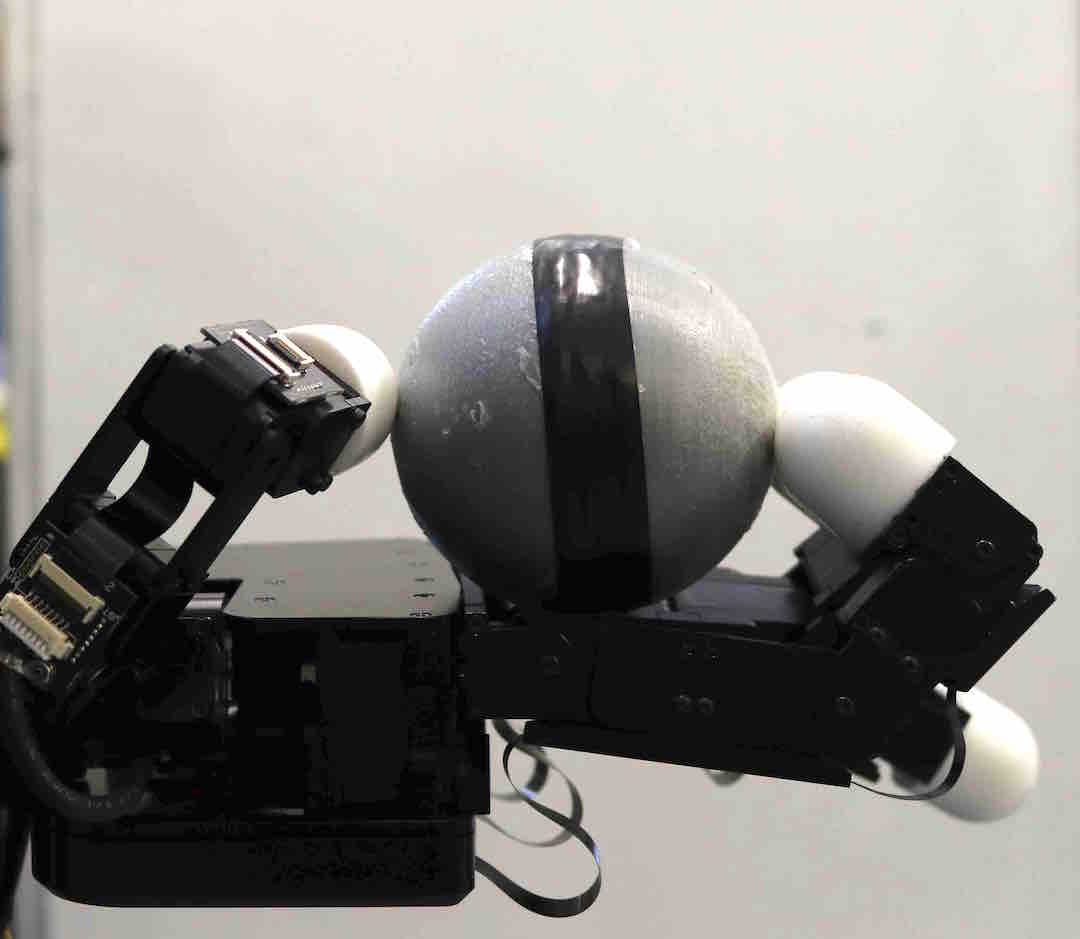}}
	\subcaptionbox{ Final configuration \label{fig.demo 4 final sphere} }
		{\includegraphics[scale=.1063]{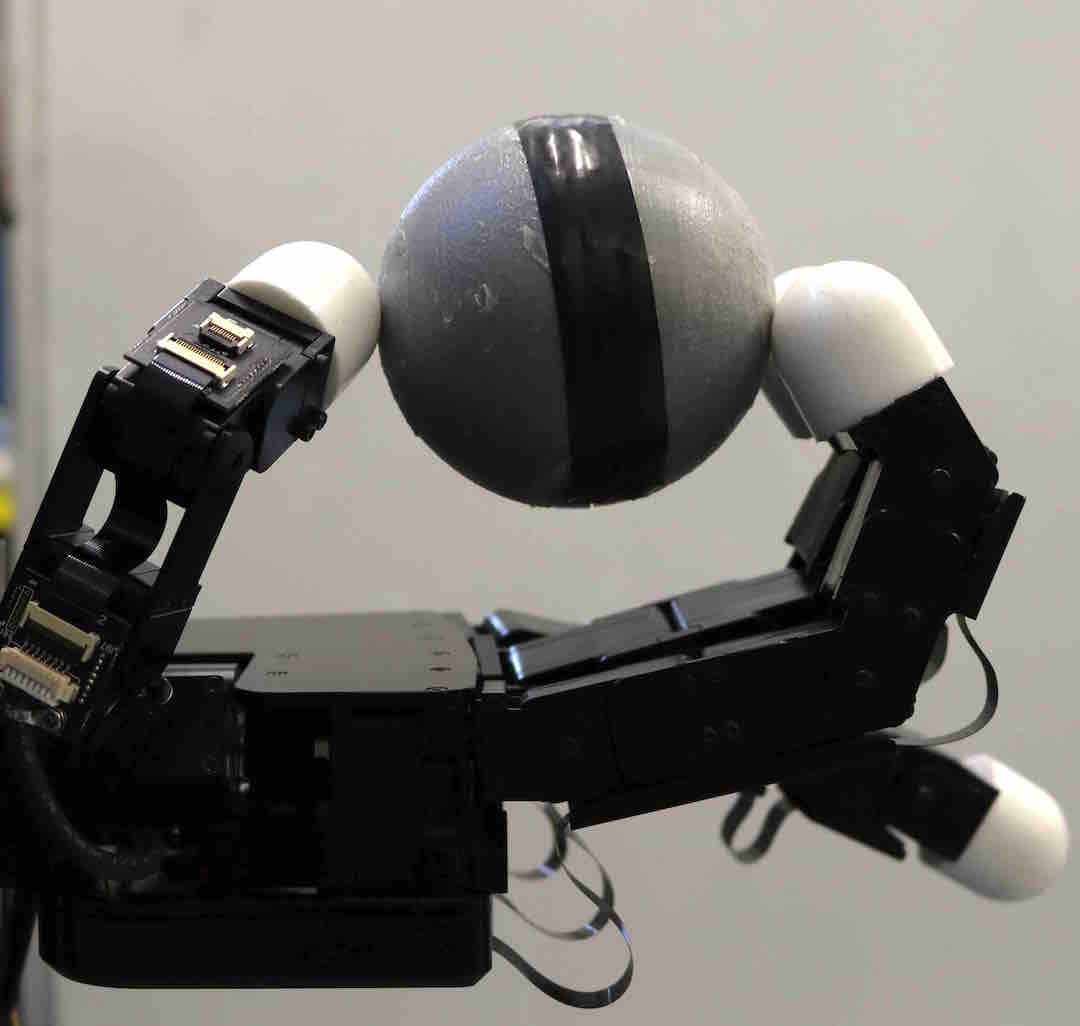}}
	\subcaptionbox{Position error.  \label{fig:cs4_position} }
		{\includegraphics[scale=.222]{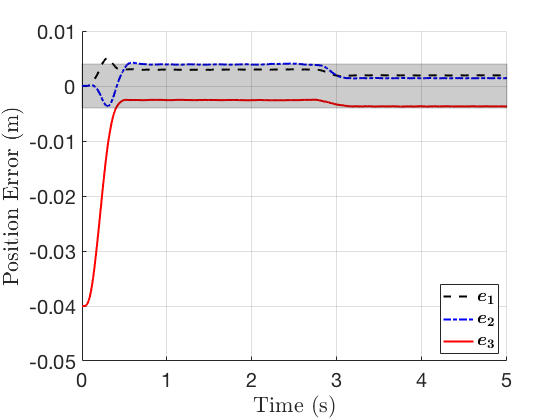}}
	\subcaptionbox{Orientation error.  \label{fig:cs4_orientation} }
		{\includegraphics[scale=.222]{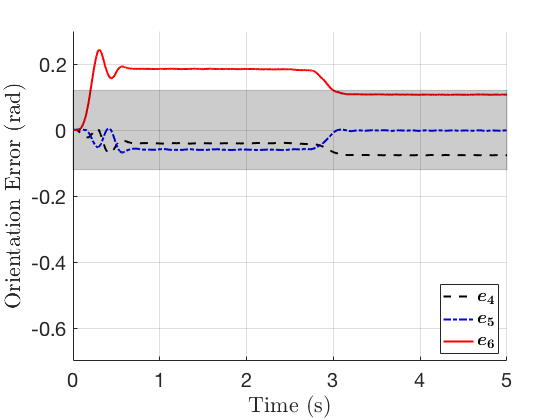}}
	\subcaptionbox{Control torque including dither signal (dither is shut off at 3.00 s when tolerance is reached).   \label{fig:cs4_torque} }
		{\includegraphics[scale=.222]{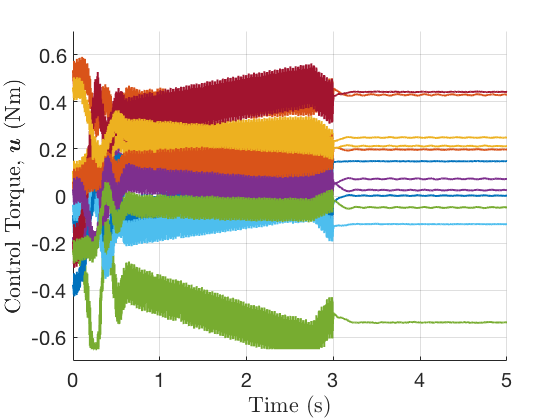}}
	\caption{Demonstration 4: Translation of $0.20$ kg sphere to $0.04$ m.}  \label{fig:cs4}
\end{figure}

Figure \ref{fig:cs4} shows the final demonstration, in which the $0.20$ kg spherical object is translated. Figure \ref{fig:cs4_position} shows the position error as the robotic hand pushes the object along the $Z$-direction to the reference of $0.04$ m. Figure \ref{fig:cs4_orientation} shows that the controller maintains the initial orientation of the object, within tolerance, during the translation maneuver. Note in Figure \ref{fig:cs4_torque} minor motor saturation occurs as additional torque is required to act against gravity to move the object. This final demonstration shows that the same control is capable of performing both rotation and translation commands despite mass uncertainty and without re-tuning of the controller gains.

The demonstrations highlight the robustness of the proposed controller in manipulating objects of different mass and shape using only $\myvar{q}$, $\myvardot{q}$ measurements. The results validate the stability analysis of Section \ref{sec:proposed solution} which considers uncertainties in $M_{ho}$, $C_{ho}$, $\myvar{w}_e$, $\myvar{\tau}_e$, $G$, and $J_h$, in addition to rolling effects.

\section{Conclusion} \label{sec: conclusion}

In this paper a robust control law for in-hand manipulation was proposed to handle disturbances that manifest from unknown/uncertain hand-object model, external wrenches, and contact locations. The proposed control only requires joint angle measurements for applicability to tactile-based blind grasping as well as vision-based systems. The stability analysis provides semi-global asymptotic and exponential stability about a set pose reference point, and tuning guidelines. Simulation and hardware results demonstrate the efficacy of the proposed control law. 

\section*{Appendix}

\begin{table}[H]\hspace*{-1cm}
\centering
\caption{Dither Signal Parameters ($d_j = a_j \sin(2\pi f t) + b_j, j\in [1,m]$, $f = 150$ Hz)} \label{table:dither parameters}
\begin{tabular}{c|cccc}    \toprule
Index: & \emph{Joint 1} & \emph{Joint 2} & \emph{Joint 3} & \emph{Joint 4} \\\midrule
$a_j$   & $0.0990$ & $0.0873$ &  $0.0936$  & $0.0828$ \\
$b_j$    & $-0.0110$ & $-0.0255 $ & $-0.0260$ & $-0.0140$  \\
\midrule
Middle: & \emph{Joint 1} & \emph{Joint 2} & \emph{Joint 3} & \emph{Joint 4} \\\midrule
$a_j$   & $0.0981$ & $0.0639$ &  $0.0747$  &$ 0.0648$ \\
$b_j$    & $-0.0175 $ & $-0.0045 $ & $-0.0065$ & $-0.0040$  \\
\midrule
Thumb: & \emph{Joint 1} & \emph{Joint 2} & \emph{Joint 3} & \emph{Joint 4} \\\midrule
$a_j$   & $0.1152$ & $0.0720$ &  $0.1224$  & $0.0648$\\
$b_j$    & $-0.0250 $ & $-0.0130 $ & $-0.0180$ & $0.0010$  \\
\bottomrule
\end{tabular}
\end{table}

\section*{Acknowledgement}
This work is supported by the Valma Angliss Trust.

\section*{References}

\bibliographystyle{elsarticle-num} 
\bibliography{IEEEabrv,ShawCortez_JournalCEP}

\end{document}